\documentclass[12pt]{imsart}

\usepackage{amsmath, amssymb, mathtools, amsthm}
\usepackage{bm} 
\usepackage[colorlinks,citecolor=blue,urlcolor=blue,filecolor=blue,backref=page]{hyperref}
\usepackage{newtxtext, newtxmath}
\usepackage{microtype} 
\usepackage{graphicx} 
\usepackage{natbib}
\usepackage{cleveref}
\pubyear{2025}
\volume{0}
\issue{0}
\firstpage{1}
\lastpage{29}

\startlocaldefs
\newtheorem{theorem}{Theorem}[section]
\newtheorem{lemma}[theorem]{Lemma}

\newtheorem{conjecture}[theorem]{Conjecture}
\theoremstyle{definition}
\newtheorem{definition}[theorem]{Definition}
\newtheorem{assumption}[theorem]{Assumption}
\newtheorem{remark}[theorem]{Remark}

\newcommand{\R}{\mathbb{R}}
\newcommand{\E}{\mathbb{E}}
\newcommand{\K}{\mathcal{K}}
\newcommand{\A}{\mathcal{A}}
\newcommand{\M}{\mathcal{M}} 
\newcommand{\hamiltonian}{\mathcal{H}}
\DeclareMathOperator{\Tr}{Tr}

\DeclarePairedDelimiter{\abs}{\lvert}{\rvert}
\DeclarePairedDelimiter{\norm}{\lVert}{\rVert}

\endlocaldefs

\begin{document}

\begin{frontmatter}
\title{Neural Hamiltonian Operator}
\runtitle{Neural Hamiltonian Operator}
\begin{aug}
  \author{\fnms{Qian} \snm{Qi}\thanksref{t1}\ead[label=e1]{qiqian@pku.edu.cn}}

  \thankstext{t1}{Peking University, Beijing, 100871, P.R.China.}

  \runauthor{Q. Qi}

  \address{Peking University, Beijing, 100871, P.R.China.\\
          \printead{e1}}
\end{aug}

    \begin{abstract}
    Stochastic control problems in high dimensions are notoriously difficult to solve due to the curse of dimensionality. An alternative to traditional dynamic programming is Pontryagin's Maximum Principle (PMP), which recasts the problem as a system of Forward-Backward Stochastic Differential Equations (FBSDEs). In this paper, we introduce a formal framework for solving such problems with deep learning by defining a \textbf{Neural Hamiltonian Operator (NHO)}. This operator parameterizes the coupled FBSDE dynamics via neural networks that represent the feedback control and an ansatz for the value function's spatial gradient. We show how the optimal NHO can be found by training the underlying networks to enforce the consistency conditions dictated by the PMP. By adopting this operator-theoretic view, we situate the deep FBSDE method within the rigorous language of statistical inference, framing it as a problem of learning an unknown operator from simulated data. This perspective allows us to prove the universal approximation capabilities of NHOs under general martingale drivers and provides a clear lens for analyzing the significant optimization challenges inherent to this class of models.
    \end{abstract}



\end{frontmatter}

\section{Introduction}

The optimal control of stochastic dynamical systems is a central problem in modern science and engineering. The theory rests on two pillars: the Hamilton-Jacobi-Bellman (HJB) equation, a nonlinear partial differential equation for the value function, and Pontryagin's Maximum Principle (PMP), which yields a coupled Forward-Backward Stochastic Differential Equation (FBSDE) system. While HJB-based grid methods suffer from the curse of dimensionality, the PMP framework offers a scalable alternative by avoiding state-space discretization.

Recent advances in deep learning have provided a powerful toolkit for solving the high-dimensional FBSDEs arising from the PMP \cite{HanJentzenE2018}. The core idea is to parameterize the unknown functions in the FBSDE system, namely the optimal control and the adjoint processes, with neural networks. 

In this paper, we unify this computational approach under a single mathematical object: the \textbf{Neural Hamiltonian Operator (NHO)}. Drawing the inspirations of the recent universal approximation theoretical guarantees of deep reinforcement learning on viscosity solutions (see \cite{qi2025universalappr,qian2025icml}), we reframe the problem from a purely algorithmic perspective to one of statistical inference The objective becomes to infer an unknown operator, the true generator of the optimal dynamics, from data. The data are trajectories simulated under a parameterized model class of NHOs, and the inference is guided by minimizing an empirical risk functional that penalizes violations of a known physical consistency condition (the terminal constraint from the PMP). This operator-theoretic perspective allows us to:
\begin{enumerate}
    \item Provide a formal, self-contained definition of the learning problem in stochastic control, generalized to systems driven by continuous martingales.
    \item Frame the training algorithm as a search for an optimal NHO by minimizing an M-estimation objective that enforces the terminal boundary conditions of the PMP.
    \item Rigorously situate the known analytical challenges (approximation, optimization, well-posedness) as fundamental questions about the properties of the space of NHOs.
    \item Prove that the family of NHOs is dense in the space of true Hamiltonian operators, establishing the theoretical soundness of the approach.
\end{enumerate}
Our work thus provides a rigorous mathematical footing for this promising computational method, clarifying its structure and the frontiers of the open questions it entails from both a probabilistic and statistical viewpoint.

\subsection{Related Literature}

The development of numerical methods for high-dimensional stochastic control and PDEs has a rich history, recently invigorated by advances in deep learning. Our work is situated at the confluence of several research streams.

\paragraph{Dynamic Programming and HJB Equations.} The classical approach to stochastic control is through dynamic programming, leading to the Hamilton-Jacobi-Bellman (HJB) equation (e.g., \cite{Fleming2006}). While grid-based methods for solving the HJB equation are effective in low dimensions, they suffer from the curse of dimensionality, making them computationally infeasible for the problems we consider. Modern theory relies on the concept of viscosity solutions (see \cite{Crandall1992}), which guarantees the existence and uniqueness of a solution even when it is not classically differentiable. A significant line of research in machine learning has focused on directly approximating the value function $V(t,s)$ and solving the HJB equation as a PDE. Methods like Deep Galerkin Methods (DGM, see \cite{Sirignano2018}) and Physics-Informed Neural Networks (PINNs, see \cite{Raissi2019}) learn the value function by minimizing the PDE residual over a set of sampled collocation points. Our work differs fundamentally by operating within the PMP framework rather than directly on the second-order HJB equation. While we discuss the connection to viscosity solutions in Conjecture \ref{conj:visco}, our primary methodology avoids the explicit computation of second-order derivatives (Hessians) of the neural network, which is a known source of optimization challenges in PINN-style methods.

\paragraph{Pontryagin's Maximum Principle and FBSDEs.} An alternative to the HJB framework is Pontryagin's Maximum Principle (PMP), which characterizes the optimal control via a system of coupled Forward-Backward Stochastic Differential Equations (FBSDEs, see \cite{Yong1999}). This approach avoids state-space discretization and is naturally suited to high-dimensional problems. The challenge is then shifted to solving the FBSDE system. Early numerical methods for FBSDEs were based on multi-step schemes and regression (e.g., \cite{Bender2008}), but these also faced scalability issues.

The breakthrough for high-dimensional problems came with the introduction of deep learning-based solvers for BSDEs and FBSDEs. The seminal work of \cite{EHanJentzen2017, HanJentzenE2018} proposed parameterizing the unknown decoupling field of the BSDE with a neural network and training it to satisfy the terminal condition via a "shooting method" approach. This core idea forms the computational foundation of our NHO framework. Several subsequent works have extended and refined this deep FBSDE method, for example, by using multi-step schemes (see \cite{Beck2019}) or by exploring different network architectures.

\paragraph{Operator Learning and Our Contribution.} Our work reframes the deep FBSDE method through the lens of operator learning. This field, which aims to learn mappings between infinite-dimensional function spaces, has seen rapid development with the introduction of "Neural Operators" (see \cite{Kovachki2023NeuralOperatorReview}). Architectures like the Fourier Neural Operator (FNO, see \cite{Li2021FNO}) and the DeepONet (see \cite{Lu2021DeepONet}) have proven effective at learning solution operators for families of PDEs, often by parameterizing integral kernels in a learned, resolution-invariant manner. These methods typically require a dataset of input-output pairs (e.g., pairs of initial conditions and corresponding final solutions) to learn the mapping.

Our NHO framework shares the goal of learning an operator but differs in two fundamental ways. First, instead of learning the solution operator that maps initial states to final states, the NHO learns the \textit{infinitesimal generator} of the system's dynamics. This is a more fundamental object that implicitly defines the solution operator for all time horizons. Second, the NHO is not trained on a pre-existing dataset of solved problem instances. Instead, it is trained via a self-supervised objective derived from a physical consistency condition, the PMP's terminal constraint. The data are generated on-the-fly by the operator's own evolving dynamics. By defining a specific mathematical object, the \textbf{Neural Hamiltonian Operator}, that represents this generator, we elevate the discussion from a purely algorithmic one to a statistical inference problem: learning an unknown dynamical law from simulated trajectories under physical constraints. This specialized structure, and its generalization to systems driven by continuous martingales, distinguishes our contribution from both prior deep FBSDE methods and general-purpose neural operators. By providing this rigorous mathematical language, we aim to build a stronger bridge between the computational methods of machine learning and the theoretical foundations of stochastic analysis and control theory.

The paper is organized as follows. In \Cref{sec:framework}, we review the stochastic control problem and Pontryagin's Maximum Principle. In \Cref{sec:nho_method}, we introduce our central contribution, the Neural Hamiltonian Operator (NHO), and frame the learning problem. We then establish the theoretical properties of NHOs, proving their universal approximation power and analyzing their convergence in \Cref{sec:approximation}. We discuss extensions to infinite-horizon problems and the connection to viscosity solutions in \Cref{sec:extensions}. We present extensive numerical validations in \Cref{sec:numerics} before concluding in \Cref{sec:conclusion}. Proofs are deferred to the appendices.

\section{Theoretical Framework}
\label{sec:framework}

\subsection{The Stochastic Control Problem}
Let $(\Omega, \mathcal{F}, \{\mathcal{F}_t\}_{t \in [0,T]}, \mathbb{P})$ be a filtered probability space supporting a $d_M$-dimensional continuous, square-integrable martingale $\M_t$ with $\M_0 = 0$. We assume its quadratic variation process, $\langle \M \rangle_t$, is absolutely continuous with respect to the Lebesgue measure, i.e., there exists a predictable, symmetric, and positive semi-definite matrix-valued process $C_t \in \R^{d_M \times d_M}$ such that $d\langle \M \rangle_t = C_t dt$. The canonical example is a standard $d_M$-dimensional Brownian motion $\mathcal{Z}_t$, for which $C_t = I_{d_M}$.

For a given initial time $t \in [0,T)$, an admissible control $\alpha = \{\alpha_u\}_{u \in [t,T]}$ is a process that is predictable with respect to $\{\mathcal{F}_u\}$, takes values in a compact set $\K \subset \R^{m}$, and satisfies $\E[\int_t^T \abs{\alpha_u}^2 du] < \infty$. We denote the set of all such admissible controls for a given start time $t$ by $\A_t$. For a given initial state $s \in \R^d$, the state process $S_u \in \R^d$ for $u \in [t, T]$ evolves according to the stochastic differential equation (SDE):
\begin{equation}
\label{eq:state_sde}
dS_u = \mu(u, S_u, \alpha_u)du + \sigma(u, S_u, \alpha_u) d\M_u, \quad S_t = s,
\end{equation}
where $\sigma: [0,T] \times \R^d \times \K \to \R^{d \times d_M}$. The objective is to find the value function $V(t,s) = \sup_{\alpha \in \A_t} \E_{t,s} \left[ \int_t^T f(u, S_u, \alpha_u) du + G(S_T) \right]$.

\begin{assumption}[Regularity Conditions]
\label{ass:main_assumptions}
The functions $\mu: [0,T] \times \R^d \times \K \to \R^d$, $\sigma: [0,T] \times \R^d \times \K \to \R^{d \times d_M}$, and $f: [0,T] \times \R^d \times \K \to \R$ are continuously differentiable with respect to $(s, \alpha)$, and $G: \R^d \to \R$ is continuously differentiable with respect to $s$. Furthermore, these functions and their derivatives with respect to $s$ are Lipschitz continuous in $s$, uniformly in all other arguments. The set $\K$ is compact and convex. The quadratic variation process $C_t$ is assumed to be uniformly bounded.
\end{assumption}

\subsection{Pontryagin's Maximum Principle and the Hamiltonian System}
The PMP introduces the Hamiltonian $\hamiltonian: [0,T] \times \R^d \times \K \times \R^d \times \R^{d \times d_M} \to \R$, a deterministic function defined as
$$ \hamiltonian(t,s,\alpha,p,q) \coloneqq \mu(t,s,\alpha)^\top p + \Tr\left(\sigma(t,s,\alpha)^\top q\right) + f(t,s,\alpha). $$
The PMP asserts that for an optimal pair $(\alpha^*_t, S^*_t)$, there exists an adapted solution $(p_t, q_t)$ to the Backward Stochastic Differential Equation (BSDE):
\begin{equation}
\label{eq:adjoint_bsde}
-dp_t = \nabla_s \hamiltonian(t, S^*_t, \alpha^*_t, p_t, q_t) dt - q_t d\M_t, \quad p_T = \nabla G(S^*_T),
\end{equation}
such that the Hamiltonian is maximized almost surely for almost every $t$: $\hamiltonian(t, S^*_t, \alpha^*_t, p_t, q_t) = \max_{\alpha \in \K} \hamiltonian(t, S^*_t, \alpha, p_t, q_t)$.
If the value function $V(t,s)$ is a classical $C^{1,2}$ solution to the HJB equation, then the adjoint processes are identified as its derivatives: $p_t = \nabla_s V(t, S^*_t)$. The predictable process $q_t$ is then identified via the martingale representation theorem. Specifically, assuming $V$ is a $C^{1,2}$ function, applying Itô's formula to the process $p_t = \nabla_s V(t, S^*_t)$ and comparing the resulting martingale term with that in \eqref{eq:adjoint_bsde} yields the identification $q_t = (\nabla_s^2 V(t, S^*_t)) \sigma(t, S^*_t, \alpha^*_t)$ [cf. \cite{Yong1999}, Chapter 4].

\begin{remark}[On the Reliance on Classical Solutions]
The identification of the adjoint processes $(p_t, q_t)$ with derivatives of the value function is central to the NHO construction. This is a deliberate modeling choice. We acknowledge that the modern theory of HJB equations relies on weaker viscosity solutions, which are not guaranteed to be differentiable. The NHO framework is therefore motivated by a setting where classical solutions are assumed to exist. The resulting computational method can, however, be viewed as a numerical scheme in its own right. The question of whether this scheme converges even when classical solutions do not exist, and if so, to the correct viscosity solution, is a deep and important one, which we revisit in Conjecture \ref{conj:visco}.
\end{remark}

\section{The Neural Hamiltonian Operator (NHO) Method}
\label{sec:nho_method}
We now formalize the deep learning approach to solving the PMP. The central idea is to parameterize the unknown feedback control and the gradient of the value function with neural networks. This parameterization defines a family of candidate Hamiltonian systems. The learning task is to select the system whose dynamics are consistent with the terminal conditions of the PMP. We encapsulate the generator of these dynamics in a single mathematical object.

\subsection{The Parameterized Hamiltonian System}
Let $\alpha_\omega: [0,T] \times \R^d \to \K$ be a neural network with parameters $\omega$ approximating the optimal feedback control $\alpha^*(t,s)$. Let $\Phi_\xi: [0,T] \times \R^d \to \R^d$ be a neural network with parameters $\xi$ that serves as an ansatz for the decoupling field, i.e., $\Phi_\xi(t,s) \approx \nabla_s V(t,s)$. We denote the full set of trainable parameters by $\Psi = (\omega, \xi)$.

The function $\Phi_\xi$ allows us to construct a candidate for the process $q_t$. This is motivated by It\^o's formula; it is precisely the form that the process $q_t$ would take if the adjoint process $p_t$ were given by the decoupling field ansatz, $p_t = \Phi_\xi(t, S_t)$. Here, $\nabla_s \Phi_\xi$ denotes the Jacobian matrix of the network output with respect to its state input.
We define our approximation $q_\Psi(t,s) \in \R^{d \times d_M}$ by differentiating the network $\Phi_\xi$:
$$ q_{\Psi}(t,s) \coloneqq \left(\nabla_s \Phi_\xi(t,s)\right) \sigma(t,s,\alpha_\omega(t,s)). $$

For a given parameter set $\Psi$, we replace the coupled forward-backward system with a candidate, fully-specified forward SDE system. We consider the $2d$-dimensional process $\mathcal{X}_t^\Psi = (S_t, \tilde{p}_t)$, whose evolution is governed by:
\begin{align}
\label{eq:fwd_sde_psi}
dS_t &= \mu(t, S_t, \alpha_\omega(t,S_t)) dt + \sigma(t, S_t, \alpha_\omega(t,S_t)) d\M_t, \\
\label{eq:bwd_sde_psi}
d\tilde{p}_t &= -\nabla_s \hamiltonian(t, S_t, \alpha_\omega(t,S_t), \Phi_\xi(t,S_t), q_\Psi(t,S_t)) dt + q_\Psi(t,S_t) d\M_t.
\end{align}
This constitutes a family of forward SDE systems parameterized by $\Psi$. The process $\tilde{p}_t$ represents our running estimate of the true adjoint process $p_t$. The goal of learning is to find a $\Psi^*$ such that the solution to this forward system satisfies the terminal condition required by the PMP.

\subsection{Formal Definition of the Neural Hamiltonian Operator}
The dynamics of the system \eqref{eq:fwd_sde_psi}-\eqref{eq:bwd_sde_psi} can be described by an infinitesimal generator. This leads to the central definition of our framework.

\begin{definition}[The Neural Hamiltonian Operator]
\label{def:nho_revised}
Let $\Psi = (\omega, \xi)$ be a set of network parameters. The associated dynamics for the extended state process $\mathcal{X}_t = (S_t, \tilde{p}_t)$ are defined by the drift vector $b_\Psi: [0,T] \times \R^{d} \to \R^{2d}$ and the diffusion coefficient matrix $\Sigma_\Psi: [0,T] \times \R^{d} \to \R^{2d \times d_M}$ as follows, where we decompose $b_\Psi = (b_{\Psi,s}, b_{\Psi,p})^\top$:
\begin{align*}
b_{\Psi,s}(t, s) &\coloneqq \mu(t, s, \alpha_\omega(t,s)), \\
b_{\Psi,p}(t, s) &\coloneqq - \nabla_s \hamiltonian(t, s, \alpha_\omega(t,s), \Phi_\xi(t,s), q_\Psi(t,s)), \\
\Sigma_\Psi(t, s) &\coloneqq
\begin{pmatrix}
\sigma(t, s, \alpha_\omega(t,s)) \\
q_\Psi(t,s)
\end{pmatrix}.
\end{align*}
The diffusion of the SDE system is determined by the $2d \times 2d$ matrix $\mathcal{D}_\Psi(t,s) \coloneqq \Sigma_\Psi(t,s) C_t \Sigma_\Psi(t,s)^\top$. The \textbf{Neural Hamiltonian Operator (NHO)} $L_\Psi$ is the second-order partial differential operator associated with the spatial components of this SDE system. For a fixed time $t$, its action on a suitable test function $g \in C^{2}(\R^{d} \times \R^d; \R)$ at a state $x=(s,y) \in \R^d \times \R^d$ is given by:
\[
(L_\Psi g)(t, x) = \nabla_x g(x)^\top b_\Psi(t,s) + \frac{1}{2} \Tr\left(\mathcal{D}_\Psi(t,s) \nabla_x^2 g(x)\right),
\]
where $\nabla_x g = (\nabla_s g, \nabla_y g)^\top$ and $\nabla_x^2 g$ is the Hessian of $g$ with respect to the full state variable $x=(s,y)$. The full evolution is governed by the parabolic operator $(\partial_t + L_\Psi)$. Note that the operator's coefficients depend only on $(t,s)$, not on the full extended state $(t,s,y)$.
\end{definition}

\begin{remark}[Degeneracy and Hypoellipticity]
The operator $L_\Psi$ is the infinitesimal generator of the SDE system for $\mathcal{X}_t=(S_t, \tilde{p}_t)$. It is crucial to observe that the coefficients $b_\Psi(t,s)$ and $\Sigma_\Psi(t,s)$ depend only on the $s$-component of the full state $x=(s,y)$. This structure correctly reflects that the evolution of the candidate adjoint process $\tilde{p}_t$ is driven by the state process $S_t$, not by its own value, because the control $\alpha_\omega$ and the decoupling field ansatz $\Phi_\xi$ are defined as feedback functions of the state.

The operator $L_\Psi$ is a \textbf{degenerate elliptic} operator. The diffusion matrix from Definition \ref{def:nho_revised} can be written in block form (with arguments suppressed for clarity):
$$ \mathcal{D}_\Psi(t,s) = \begin{pmatrix} \sigma C_t \sigma^\top & \sigma C_t q_\Psi^\top \\ q_\Psi C_t \sigma^\top & q_\Psi C_t q_\Psi^\top \end{pmatrix}. $$
Since both components of the system are driven by the same $d_M$-dimensional martingale $\M_t$, the rank of this $2d \times 2d$ matrix is at most the rank of $C_t$, which is at most $d_M$. For any non-trivial problem where $d_M < 2d$, the operator is degenerate. Such operators are often hypoelliptic, meaning that even though the operator is degenerate, solutions to the associated parabolic PDE $( \partial_t + L_\Psi)u = 0$ can be smoother than the initial data. This property, explored in contexts like Hörmander's theorem, is fundamental to the regularity of the underlying process, although it is not guaranteed and depends on the Lie algebra generated by the system's vector fields satisfying a full-rank condition.
\end{remark}

\subsection{The Learning Problem as an Operator Search}
The learning algorithm can now be precisely stated as a search for the optimal operator $L_{\Psi^*}$ in the parameterized family $\{L_\Psi\}_{\Psi}$. The optimal operator is the one whose associated dynamics satisfy the terminal condition of the PMP. Let the time interval $[0,T]$ be discretized as $0=t_0 < \dots < t_N = T$. For a given initial state $s_0$:
\begin{enumerate}
    \item \textbf{Initialization:} The true adjoint process $p_t$ is characterized by a terminal condition at $t=T$. Our simulation, however, must run forward in time. We initialize the joint process $\mathcal{X}_{t_0} = (S_{t_0}, \tilde{p}_{t_0})$ using the network $\Phi_\xi$ as an ansatz for the initial value of the adjoint process:
    $$ S_{t_0} = s_0, \quad \tilde{p}_{t_0} = \Phi_\xi(t_0, s_0). $$
    
    \begin{remark}[On the Initialization Strategy]
    This initialization is a critical design choice: it makes the entire forward trajectory dependent on the parameters $\xi$ from the very beginning. This dependence is precisely what allows the minimization of a terminal error to propagate back and inform the choice of the ansatz network $\Phi_\xi$. This contrasts with earlier methods like the original deep BSDE algorithm (see \cite{EHanJentzen2017}), where the initial value of the adjoint process was often treated as a separate learnable parameter vector. The chosen approach directly ties the initial condition to the function approximation goal for the decoupling field, enforcing a globally consistent representation.
    \end{remark}

    \item \textbf{Trajectory Generation:} Evolve the joint process $\mathcal{X}_t = (S_t, \tilde{p}_t)$ forward by simulating the SDE whose generator is $L_\Psi$. Using a generalized Euler-Maruyama scheme with time step $\Delta t_i = t_{i+1} - t_i$ and martingale increment $\Delta\M_i = \M_{t_{i+1}} - \M_{t_i}$, for $i = 0, \dots, N-1$:
    $$ \mathcal{X}_{t_{i+1}} = \mathcal{X}_{t_i} + \begin{pmatrix} b_{\Psi,s}(t_i, S_{t_i}) \\ b_{\Psi,p}(t_i, S_{t_i}) \end{pmatrix} \Delta t_i + \Sigma_\Psi(t_i, S_{t_i}) \Delta\M_i. $$
    The increment $\Delta\M_i$ has a conditional mean of zero and conditional covariance $\E[\Delta\M_i \Delta\M_i^\top | \mathcal{F}_{t_i}] \approx C_{t_i} \Delta t_i$. In practice, this is often simulated as a Gaussian random variable $\sqrt{\Delta t_i} C_{t_i}^{1/2} \bm{\epsilon}_i$ where $\bm{\epsilon}_i \sim N(0, I_{d_M})$.

    \item \textbf{Enforcing the Terminal Condition:} The PMP requires $p_T = \nabla G(S_T)$. We enforce this condition on our simulated trajectory by minimizing a loss function that penalizes the mismatch at terminal time. This objective is an empirical risk functional:
    \begin{equation}\label{eq:loss}
        \mathcal{J}(\Psi) = \mathcal{J}(\omega, \xi) = \E \left[ \norm*{ \tilde{p}_{T} - \nabla G(S_{T}) }^2 \right].
    \end{equation}
\end{enumerate}
The expectation is approximated via Monte Carlo averaging over a batch of trajectories. Minimizing $\mathcal{J}(\Psi)$ with stochastic gradient descent corresponds to searching for a parameter set $\Psi^*$ such that the operator $L_{\Psi^*}$ generates dynamics consistent with the necessary conditions of optimality.

\section{Approximation Theory of Neural Hamiltonian Operators}\label{sec:approximation}

The entire NHO methodology hinges on the assumption that the parameterized family of operators $\{L_\Psi\}$ is rich enough to approximate the true Hamiltonian dynamics. We now prove that this is indeed the case. This requires a stronger result than standard universal approximation theorems, as we must approximate not just a function but also its spatial derivatives.

Let the optimal control problem admit a unique, classical solution $(\alpha^*, S^*, p^*, q^*)$ where $\alpha^*(t,s)$ and the decoupling field $p^*(t,s) = \nabla_s V(t,s)$ are $C^1$ functions. The true Hamiltonian operator, denoted $L^*$, is defined by the coefficients $b^*$ and $\Sigma^*$ derived from these optimal quantities:
\begin{align*}
b^*(t, s) &\coloneqq 
\begin{pmatrix}
\mu(t, s, \alpha^*(t,s)) \\
- \nabla_s \hamiltonian(t, s, \alpha^*(t,s), p^*(t,s), q^*(t,s))
\end{pmatrix}, \\
\Sigma^*(t, s) &\coloneqq
\begin{pmatrix}
\sigma(t, s, \alpha^*(t,s)) \\
q^*(t,s)
\end{pmatrix},
\end{align*}
where $q^*(t,s)$ is such that $q^*(t,s) = (\nabla_s p^*(t,s)) \sigma(t,s,\alpha^*(t,s))$. The central question of approximation is whether we can find an NHO arbitrarily close to $L^*$. The following theorem answers this in the affirmative.

\begin{theorem}[Universal Approximation Power of NHOs]
\label{thm:approx}
Let the optimal solution functions $\alpha^*(t,s)$ and $p^*(t,s)=\nabla_s V(t,s)$ be continuously differentiable on $[0,T] \times \mathcal{D}$ for some compact domain $\mathcal{D} \subset \R^d$. Let the neural network architectures for $\alpha_\omega$ and $\Phi_\xi$ be sufficiently large (e.g., in width) and use a smooth, non-polynomial activation function. Then, for any $\epsilon > 0$, there exists a parameter set $\Psi = (\omega, \xi)$ such that the coefficients of the NHO $L_\Psi$ are uniformly close to the coefficients of the true operator $L^*$:
$$
\sup_{t \in [0,T], s \in \mathcal{D}} \norm*{ b_\Psi(t,s) - b^*(t,s) } + \sup_{t \in [0,T], s \in \mathcal{D}} \norm*{ \Sigma_\Psi(t,s) - \Sigma^*(t,s) }_F < \epsilon.
$$
where $\norm*{\cdot}_F$ is the Frobenius norm.
\end{theorem}
\begin{proof}
The proof relies on reducing the approximation of the operator coefficients to the simultaneous approximation of the functions $\alpha^*$, $p^*$, and the gradient $\nabla_s p^*$ by neural networks. This property, often called $C^1$-approximation, is established in Lemma \ref{lem:c1_uat_full} in Appendix \ref{app:c1_proof}. The core logic is independent of the nature of the stochastic driver $\M_t$, as it concerns the approximation of deterministic functions that define the SDE coefficients.

The argument proceeds by showing that if we can make the errors $\norm*{\alpha_\omega - \alpha^*}_{C^0}$, $\norm*{\Phi_\xi - p^*}_{C^0}$, and $\norm*{\nabla_s \Phi_\xi - \nabla_s p^*}_{C^0}$ arbitrarily small, then the errors in the coefficients $\norm*{b_\Psi - b^*}$ and $\norm*{\Sigma_\Psi - \Sigma^*}$ also become arbitrarily small.

Let $K = [0,T] \times \mathcal{D}$. The functions $\mu, \sigma, f$ and their derivatives are continuous, and therefore uniformly continuous and bounded on the compact set $K \times \K$. Similarly, $\alpha^*, p^*, \nabla_s p^*$ are bounded on $K$. For any $\delta > 0$, by Lemma \ref{lem:c1_uat_full}, we can choose $\Psi = (\omega, \xi)$ such that $\sup_K \norm*{\alpha_\omega - \alpha^*} < \delta$, $\sup_K \norm*{\Phi_\xi - p^*} < \delta$, and $\sup_K \norm*{\nabla_s\Phi_\xi - \nabla_s p^*}_F < \delta$.

The error in the diffusion term $\Sigma_\Psi$ is bounded by analyzing its two block components. The error in the first component, $\norm*{\sigma(t,s,\alpha_\omega(t,s)) - \sigma(t,s,\alpha^*(t,s))}_F$, can be made arbitrarily small by the uniform continuity of $\sigma$ and the closeness of $\alpha_\omega$ to $\alpha^*$. The error in the second component, $\norm*{q_\Psi - q^*}_F$, is bounded by:
\begin{align*}
    \norm*{q_\Psi - q^*}_F &= \norm*{(\nabla_s \Phi_\xi) \sigma(\cdot,\alpha_\omega) - (\nabla_s p^*) \sigma(\cdot,\alpha^*)}_F \\
    &\le \norm*{(\nabla_s \Phi_\xi - \nabla_s p^*) \sigma(\cdot,\alpha_\omega)}_F + \norm*{(\nabla_s p^*) (\sigma(\cdot,\alpha_\omega) - \sigma(\cdot,\alpha^*))}_F \\
    &\le \norm*{\nabla_s \Phi_\xi - \nabla_s p^*}_F M_\sigma + M_{\nabla p} L_\sigma \norm*{\alpha_\omega - \alpha^*} < \delta (M_\sigma + M_{\nabla p} L_\sigma),
\end{align*}
where $M_\sigma, M_{\nabla p}$ are uniform bounds and $L_\sigma$ is a Lipschitz constant. This error vanishes as $\delta \to 0$.

Similarly, the error in the drift term $b_\Psi$ depends on the error in $\mu$ and in $\nabla_s \hamiltonian$. The error in $\mu$ is small by continuity. The Hamiltonian gradient is $\nabla_s \hamiltonian(t, s, \alpha, p, q) = (\nabla_s \mu)^\top p + \Tr((\nabla_s \sigma)^\top q) + \nabla_s f$. By Assumption \ref{ass:main_assumptions}, the derivatives $\nabla_s\mu, \nabla_s\sigma, \nabla_s f$ are continuous. Since the function $\nabla_s \hamiltonian$ is continuous in all its arguments, and we can make the arguments $(\alpha_\omega, \Phi_\xi, q_\Psi)$ uniformly close to $(\alpha^*, p^*, q^*)$, the resulting error $\norm*{\nabla_s\hamiltonian(\dots,\alpha_\omega,\Phi_\xi,q_\Psi) - \nabla_s\hamiltonian(\dots,\alpha^*,p^*,q^*)}$ can be made arbitrarily small by uniform continuity on a compact set.

Combining these bounds, the total operator coefficient error is bounded by an expression that tends to zero as $\delta \to 0$. Thus, for any $\epsilon > 0$, we can choose $\delta$ small enough, and then find parameters $\Psi$ via Lemma \ref{lem:c1_uat_full}, to ensure the total error is less than $\epsilon$.
\end{proof}

\begin{remark}[Implications of the Approximation Theorem]
Theorem \ref{thm:approx} is the theoretical cornerstone of the NHO method. It guarantees that the space of parameterized operators is rich enough to contain a representation of the true optimal dynamics, even in the general martingale setting. This transforms the problem from a question of existence into one of search: we are assured a sufficiently accurate operator exists, and the remaining challenge is to design optimization algorithms that can find it.
\end{remark}

\subsection{Convergence Analysis I: The Ideal Global Case}

Given that a sufficiently accurate NHO exists (Theorem \ref{thm:approx}), the next question is whether the optimization algorithm can find it. The learning process seeks to minimize the loss $\mathcal{J}(\Psi) = \E [ \norm*{ \tilde{p}_T - \nabla G(S_T) }^2 ]$ via stochastic gradient descent on the parameters $\Psi$. In this section, we present a convergence result under strong, idealized assumptions to illustrate the geometric properties of the loss landscape that are sufficient for global convergence.

\begin{theorem}[Global Convergence of SGD]
\label{thm:convergence}
Let the following assumptions hold:
\begin{enumerate}
    \item \textbf{(Identifiability)} The problem is well-posed such that a unique optimal operator $L^*$ exists. There is a parameter set $\Psi^*$ for which $\mathcal{J}(\Psi^*) = 0$. Furthermore, if $\mathcal{J}(\Psi)=0$ for some $\Psi$, then the coefficients of $L_\Psi$ are equal to the coefficients of $L^*$ almost everywhere on the relevant domain of the state process.
    \item \textbf{(Smoothness and Boundedness)} The loss function $\mathcal{J}(\Psi)$ is L-smooth (i.e., its gradient $\nabla \mathcal{J}$ is L-Lipschitz continuous). The stochastic gradient $\hat{g}(\Psi)$ is an unbiased estimator of $\nabla \mathcal{J}(\Psi)$ and has bounded variance: $\E[\norm*{\hat{g}(\Psi)}^2] \le M$ for all $\Psi$.
    \item \textbf{(Polyak-Łojasiewicz Condition)} The loss function $\mathcal{J}(\Psi)$ satisfies the P-L condition with constant $c > 0$: $\norm*{\nabla \mathcal{J}(\Psi)}^2 \ge 2c \mathcal{J}(\Psi)$ [cf. \cite{Polyak1963}].
\end{enumerate}
Then, stochastic gradient descent with learning rates $\gamma_k$ satisfying $\sum_{k=0}^\infty \gamma_k = \infty$ and $\sum_{k=0}^\infty \gamma_k^2 < \infty$ converges in expectation, i.e., $\lim_{k \to \infty} \E[\mathcal{J}(\Psi_k)] = 0$.
\end{theorem}

\begin{remark}[Discussion of Assumptions]
\label{rem:convergence_assumptions}
A full proof of Theorem \ref{thm:convergence} is provided in Appendix \ref{app:global_conv_proof}. The assumptions, while standard in optimization theory, are formidable to verify for neural networks and are independent of the specific stochastic driver.
\begin{itemize}
    \item \textbf{Assumption 1 (Identifiability):} This is a subtle yet crucial assumption, postulating that the map from the operator's coefficients to the terminal condition is injective. A spurious operator $L_\Psi \neq L^*$ could theoretically achieve $\mathcal{J}(\Psi)=0$ if it happens to satisfy the terminal condition by coincidence. This assumption is related to the concept of observability in control theory. Its plausibility is strengthened if the loss is defined over a rich distribution of initial states, $\mathcal{J}(\Psi) = \E_{s_0 \sim \nu}[\norm*{\tilde{p}_T^{s_0} - \nabla G(S_T^{s_0})}^2]$. Enforcing the condition for all initial states in a sufficiently large set should plausibly preclude such coincidences and identify the true dynamics.
    \item \textbf{Assumption 2 (Smoothness):} Requires that the neural network Jacobians and Hessians, which appear in the SDE coefficients, do not explode during training. This is notoriously difficult and a central challenge in the theory of deep learning.
    \item \textbf{Assumption 3 (P-L Condition):} This is a strong geometric condition on the loss landscape, replacing the need for convexity. This condition is central to much of modern statistical learning theory for non-convex problems, and proving that the NHO loss landscape possesses this property is a key open question. As we will see in Section \ref{sec:local_conv}, a more practical approach is to establish convergence under a local version of this condition, aided by regularization.
\end{itemize}
\end{remark}

\subsection{Convergence Analysis II: A Practical Framework via Regularization}
\label{sec:local_conv}

For the original finite-horizon problem, optimization remains a major challenge. The loss landscape is generally non-convex, and the strong global assumptions of Theorem \ref{thm:convergence} are unlikely to hold. To improve the local loss geometry and promote smoother solutions, we introduce a regularization term and establish a more plausible local convergence result.

We introduce a regularizer to control the complexity of the learned gradient field. The goal of this regularizer is to smooth the loss landscape in the vicinity of a minimizer, thereby making a local P-L condition more plausible. We define the regularized loss:
\begin{equation}
\label{eq:reg_loss}
\mathcal{J}_\lambda(\Psi) = \E \left[ \norm*{ \tilde{p}_{T} - \nabla G(S_{T}) }^2 \right] + \lambda \int_0^T \E \left[ \norm*{\nabla_s \Phi_\xi(t, S_t)}_F^2 \right] dt, \quad \lambda > 0.
\end{equation}
Such a regularizer penalizes oscillatory behavior in the ansatz for the value function's gradient, which can improve the conditioning of the optimization problem and potentially allow for convergence proofs under weaker assumptions.

\begin{assumption}[Local Structure]
\label{ass:local_structure}
Let $\Psi^*$ correspond to a unique optimal operator $L^*$ that minimizes the regularized loss $\mathcal{J}_\lambda$. Assume:
\begin{enumerate}
    \item \textbf{Local Smoothness:} The loss $\mathcal{J}_\lambda(\Psi)$ is twice continuously differentiable in a neighborhood $\mathcal{N}(\Psi^*)$ of $\Psi^*$.
    \item \textbf{Local P-L Condition:} There exists $c>0$ such that for $\Psi \in \mathcal{N}(\Psi^*)$, $\norm*{\nabla \mathcal{J}_\lambda(\Psi)}^2 \ge 2c (\mathcal{J}_\lambda(\Psi) - \mathcal{J}_\lambda(\Psi^*))$.
\end{enumerate}
\end{assumption}

\begin{theorem}[Local Convergence of SGD]
\label{thm:local_convergence}
Under Assumption \ref{ass:local_structure}, stochastic gradient descent on $\mathcal{J}_\lambda$ with an unbiased stochastic gradient $\hat{g}_k$ of bounded variance and a learning rate sequence $(\gamma_k)$ satisfying $\sum \gamma_k = \infty, \sum \gamma_k^2 < \infty$, if initialized within $\mathcal{N}(\Psi^*)$, converges in expectation to the optimal loss value: $\lim_{k \to \infty} \E[\mathcal{J}_\lambda(\Psi_k)] = \mathcal{J}_\lambda(\Psi^*)$.
\end{theorem}
\begin{proof} 
See Appendix \ref{app:local_conv_proof} for a detailed proof. 
\end{proof}

\section{Extensions to Infinite-Horizon Problems}
\label{sec:extensions}

We now extend the NHO framework to analyze the long-term behavior of controlled systems, specifically in the context of ergodic control, where the objective is to optimize a time-averaged cost over an infinite horizon. For the analysis in this section, we consider an autonomous (time-invariant) system. We assume the functions $\mu, \sigma, f$ do not depend on time, i.e., they are of the form $\mu(s,\alpha), \sigma(s,\alpha), f(s,\alpha)$, and the quadratic variation process is generated by a constant matrix, $d\langle\M\rangle_t = C dt$.

\subsection{The Ergodic Control Problem and the Stationary NHO}
Consider the problem of minimizing the ergodic cost:
$$ J(\alpha) = \limsup_{T \to \infty} \frac{1}{T} \E \left[ \int_0^T f(S_u, \alpha_u) du \right]. $$
The corresponding HJB equation becomes a stationary PDE, and the PMP leads to an FBSDE on $[0, \infty)$. A key feature of the solution, under appropriate assumptions, is that the value function gradient $p_t$ and the optimal control $\alpha_t$ become stationary (time-invariant) functions of the state, i.e., $p(s)$ and $\alpha(s)$. The associated Hamiltonian becomes constant, equal to the optimal ergodic cost $\lambda^*$, along the optimal trajectory. We adapt the NHO framework by using time-invariant neural networks $\alpha_\omega(s)$ and $\Phi_\xi(s)$.

\begin{definition}[Stationary NHO]
A stationary NHO is an operator $L_\Psi$ where the networks $\alpha_\omega$ and $\Phi_\xi$ are functions of state $s$ only. The associated drift and diffusion, $b_\Psi(s)$ and $\Sigma_\Psi(s)$, are also time-invariant.
\end{definition}

The learning objective is modified to enforce the constancy of the Hamiltonian. Let $\hamiltonian_\Psi(s) \coloneqq \hamiltonian(s, \alpha_\omega(s), \Phi_\xi(s), q_\Psi(s))$. A suitable loss function is based on the variance of the Hamiltonian along sample paths. In practice, the infinite horizon is approximated by a large, finite time $T$:
\begin{equation}
\label{eq:ergodic_loss}
\mathcal{J}_{ergodic}(\Psi) = \E \left[ \frac{1}{T} \int_0^T \left( \hamiltonian_\Psi(S_t) - \bar{\hamiltonian}_{\Psi,T} \right)^2 dt \right],
\end{equation}
where $\bar{\hamiltonian}_{\Psi,T} = \frac{1}{T}\int_0^T \hamiltonian_\Psi(S_u) du$. The outer expectation is over initial conditions and paths.

\begin{theorem}[Consistency of the Ergodic NHO Objective]
\label{thm:ergodic_consistency}
Let $\Psi^*$ be parameters for a stationary NHO, $L_{\Psi^*}$. Let $S_t$ be the state process generated by $L_{\Psi^*}$ starting from an initial distribution $\mu_0$. Assume that:
\begin{enumerate}
    \item \textbf{(Ergodicity)} The state process $S_t$ is ergodic with a unique invariant probability measure $\pi_{\Psi^*}$.
    \item \textbf{(Regularity)} The function $s \mapsto \hamiltonian_{\Psi^*}(s)$ is continuous.
    \item \textbf{(Optimization Success)} The parameters $\Psi^*$ achieve a global minimum of zero for the ergodic loss, i.e., $\mathcal{J}_{ergodic}(\Psi^*) = 0$ for a sufficiently large time horizon $T$.
\end{enumerate}
Then, the Hamiltonian corresponding to the learned operator is constant almost everywhere with respect to the invariant measure:
$$ \hamiltonian_{\Psi^*}(s) = \lambda^*, \quad \text{for } \pi_{\Psi^*}\text{-almost every } s, $$
where $\lambda^*$ is a constant equal to the optimal ergodic cost associated with the dynamics generated by $L_{\Psi^*}$.
\end{theorem}
\begin{proof}
A proof is provided in Appendix \ref{app:ergodic_proof}. The logic of the proof, relying on the Birkhoff Ergodic Theorem, holds for general ergodic processes, not only those driven by Brownian motion.
\end{proof}

\subsection{Enforcing Stability via Lyapunov Regularization}
For infinite-horizon problems, ensuring ergodicity is paramount. We can introduce a Lyapunov-based regularizer to promote this property. Let $U(s) = \norm*{s}^2$. Define the regularized ergodic loss:
\begin{equation}
\label{eq:ergodic_reg_loss}
\mathcal{J}_{erg,\lambda}(\Psi) = \mathcal{J}_{ergodic}(\Psi) + \lambda \E \left[ \frac{1}{T} \int_0^T (L_{\Psi,S} U)(S_t) dt \right],
\end{equation}
where $L_{\Psi,S}$ is the generator of the state process $S_t$ alone:
$$ (L_{\Psi,S} g)(s) = \nabla g(s)^\top \mu(s, \alpha_\omega(s)) + \frac{1}{2}\Tr(\sigma(s, \alpha_\omega(s)) C \sigma(s, \alpha_\omega(s))^\top \nabla^2 g(s)). $$

\begin{theorem}[Stability of the Learned Ergodic System]
\label{thm:ergodic_stability}
Consider stationary NHOs and a time-invariant quadratic variation process $C$. Assume the underlying dynamics satisfy a dissipativity condition: for some constant $K>0$ and compact set $\mathcal{C}$, $2 s^\top \mu(s,\alpha) + \Tr(\sigma(s,\alpha)C\sigma(s,\alpha)^\top) \le -K \norm*{s}^2$ for all $s \notin \mathcal{C}$ and $\alpha \in \K$. Let $\Psi^*$ be parameters that achieve a global minimum for $\mathcal{J}_{erg,\lambda}$ for $\lambda>0$. If for these parameters, the Lyapunov drift term satisfies $\E\left[\frac{1}{T} \int_0^T (L_{\Psi^*,S} U)(S_t) dt\right] < 0$, then the process $S_t$ generated by $L_{\Psi^*}$ is ergodic with a unique stationary distribution.
\end{theorem}
\begin{proof} 
See Appendix \ref{app:ergodic_stability_proof} for a detailed proof. 
\end{proof}

\begin{remark}[Challenges in the Ergodic Case]
Theorem \ref{thm:ergodic_consistency} shows that if training is successful, the learned operator satisfies a key necessary condition for optimality. However, significant challenges remain. Ergodicity (Assumption 1) is difficult to verify for a general NHO. The SDE must be simulated over a long horizon, raising stability questions. The Lyapunov regularizer provides a practical tool to promote this stability, though verifying the dissipativity condition for a given problem is non-trivial.
\end{remark}

\subsection{A Conjecture on the Connection to Viscosity Solutions}
\label{sec:conjecture}
The grand challenge is to connect the NHO framework, rooted in Pontryagin's Maximum Principle, to the modern theory of viscosity solutions for the Hamilton-Jacobi-Bellman (HJB) equation [cf. \cite{Crandall1992}, \cite{Fleming2006}]. This connection is often explored via "Physics-Informed Neural Networks" (PINNs, see \cite{Raissi2019}), which directly minimize the residual of the governing PDE.

Let $\hat{V}_\theta(t,s)$ be a neural network with parameters $\theta$ designed to approximate the value function $V(t,s)$. We obtain its partial derivatives $\partial_t \hat{V}_\theta(t,s)$, its gradient $\nabla_s \hat{V}_\theta(t,s)$, and its Hessian matrix $\nabla_s^2 \hat{V}_\theta(t,s)$ using automatic differentiation. The HJB equation for the value function $V(t,s)$, generalized for our martingale driver, is:
\begin{equation}
\label{eq:hjb_equation_standard}
-\partial_t V - \sup_{\alpha \in \K} \left\{ \mu(t,s,\alpha)^\top \nabla_s V + \frac{1}{2}\Tr\left(\sigma(t,s,\alpha)C_t\sigma(t,s,\alpha)^\top \nabla_s^2 V\right) + f(t,s,\alpha) \right\} = 0,
\end{equation}
with terminal condition $V(T,s) = G(s)$.
Let the HJB Hamiltonian be $\mathcal{L}_{\text{HJB}}(t,s,\alpha,p,P) \coloneqq \mu(t,s,\alpha)^\top p + \frac{1}{2}\Tr(\sigma(t,s,\alpha)C_t\sigma(t,s,\alpha)^\top P) + f(t,s,\alpha)$. We define the HJB residual using the neural network approximation $\hat{V}_\theta$:
\begin{equation}
\label{eq:hjb_residual_revised}
R_\theta(t,s) \coloneqq -\partial_t \hat{V}_\theta(t,s) - \sup_{\alpha \in \K} \mathcal{L}_{\text{HJB}}(t,s,\alpha,\nabla_s \hat{V}_\theta(t,s), \nabla_s^2 \hat{V}_\theta(t,s)).
\end{equation}
In practice, the supremum over $\alpha$ can be handled analytically or via adversarial training. The learning objective in this context would be to minimize the expected squared $L^2$-norm of this residual over the domain, plus a penalty for the terminal condition mismatch:
\begin{equation}
\label{eq:hjb_loss_revised}
\mathcal{J}_{visco}(\theta) = \E_{(t,s) \sim \mathcal{U}} \left[ R_\theta(t, s)^2 \right] + \beta \E_{s \sim \nu} \left[ (\hat{V}_\theta(T,s) - G(s))^2 \right],
\end{equation}
where the expectations are taken over suitable distributions for time, state, and terminal state.

\begin{conjecture}[Convergence to the Viscosity Solution]
\label{conj:visco}
Let $V$ be the unique continuous viscosity solution to the HJB equation \eqref{eq:hjb_equation_standard}. Let $\{\hat{V}_{\theta_n}\}_{n=1}^\infty$ be a sequence of value functions represented by neural networks with parameters $\{\theta_n\}_{n=1}^\infty$ corresponding to increasingly expressive architectures. If the HJB residual loss converges to zero, i.e., $\lim_{n\to\infty} \mathcal{J}_{visco}(\theta_n) = 0$, then the sequence of learned value functions $\hat{V}_{\theta_n}$ converges to the true viscosity solution $V$, uniformly on compact subsets of $[0,T) \times \R^d$.
\end{conjecture}

\begin{remark}[On Proving the Conjecture]
A proof would likely adapt the celebrated Barles-Souganidis convergence framework for numerical schemes (see \cite{Barles1991}). This requires demonstrating three properties of the scheme implied by minimizing $\mathcal{J}_{visco}(\theta)$:
\begin{enumerate}
    \item \textbf{Stability:} The family of solutions $\{\hat{V}_{\theta_n}\}$ must be uniformly equicontinuous. This is plausible if network weights and gradients are controlled, perhaps via regularization.
    \item \textbf{Consistency:} For any smooth test function $\phi$, the HJB residual $R_\theta(t,s)$ when evaluated with $\hat{V}_\theta = \phi$ must converge to zero as the network approximation error for $\phi$ and its derivatives vanishes. Minimizing $\mathcal{J}_{visco}(\theta)$ directly enforces this for the learned $\hat{V}_{\theta_n}$.
    \item \textbf{Monotonicity:} The scheme must be monotone. This is the crucial, and likely most challenging, roadblock. Monotonicity is the numerical analogue of the maximum principle that underpins viscosity solution theory. It ensures that if one numerical solution starts below another, it stays below. Neural network approximations are not, in general, monotone operators.
\end{enumerate}
The potential failure of the monotonicity property is a deep issue. A path forward might involve constructing specialized network architectures that enforce monotonicity by design (e.g., Input Convex Neural Networks \cite{Amos2017}), or developing weaker notions of convergence. The connection to viscosity solutions remains a key frontier for deep learning methods in control.
\end{remark}

\section{Numerical Validations}
\label{sec:numerics}

To demonstrate the practical efficacy and scalability of the NHO framework, we apply it to three distinct high-dimensional nonlinear control problems. For these numerical experiments, we specialize the general martingale framework to the most common case: a standard $d$-dimensional Brownian motion $\M_t = \mathcal{Z}_t$, for which $d_M = d$ and the quadratic variation matrix is the identity, $C_t = I_d$. All theoretical results apply directly to this setting. The experiments are conducted in $d=50$ dimensions to showcase the method's performance beyond trivial scales.

\subsection{Problem 1: Nonlinear Terminal Cost in High Dimensions}
We first consider a system with a known semi-analytic solution to rigorously assess accuracy. The state $S_t \in \R^d$ is governed by:
$$ dS_t = \alpha_t dt + d\mathcal{Z}_t, \quad S_0 = s_0, $$
with objective $V(t,s) = \sup_{\alpha \in \A_t} \E_{t,s} [ \int_t^T -\frac{1}{2}\norm*{\alpha_u}^2 du + G(S_T) ]$, terminal payoff $G(s) = \log( \frac{1}{2} + \frac{1}{2}\norm*{s}^2 )$, and horizon $T=1$. The analytical solution is known, and the optimal control is given by the feedback law $\alpha^*(t,s) = \nabla_s V(t,s) = p(t,s)$.

\subsubsection*{Results}
The NHO solver demonstrates exceptional accuracy in $d=50$. As shown in Figure \ref{fig:p1_validation}, the learned value function, estimated via Monte Carlo simulation, closely tracks the true analytical solution. The minor oscillations are characteristic of the MC estimation process and overlay the accurately learned mean. More importantly, the right panel shows that the learned feedback control $\alpha(0,s)$ is almost indistinguishable from the reference control. This provides strong evidence that the network $\Phi_\xi$ has converged to the true decoupling field $\nabla_s V(t,s)$. Table \ref{tab:results1} quantifies this accuracy at the origin. The convergence of the optimization is confirmed in Figure \ref{fig:p1_loss}, where the terminal loss is reduced by over two orders of magnitude and stabilizes.

\begin{figure}[htbp]
    \centering
    \includegraphics[width=\textwidth]{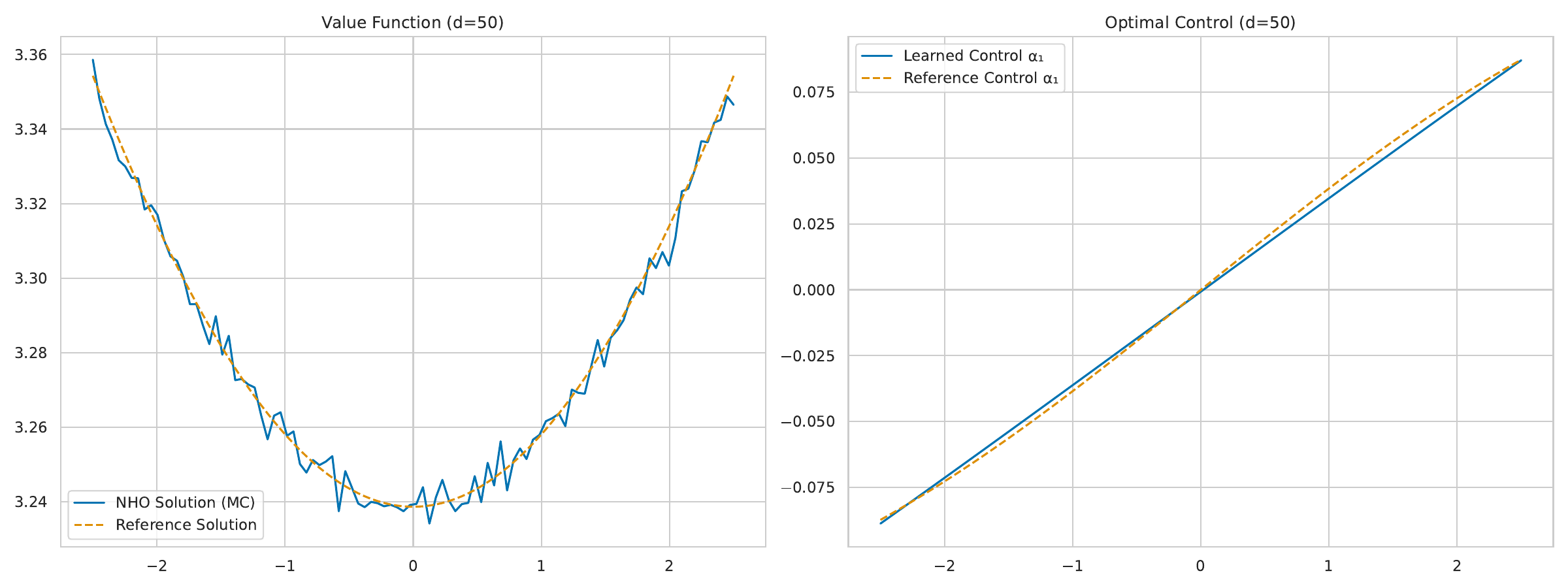} 
    \caption{Validation for Problem 1 ($d=50$). \textbf{Left:} 1-D slice of the value function $V(0,s)$. The NHO solution (blue, via Monte Carlo) accurately tracks the analytical reference solution (orange dashes). \textbf{Right:} The learned optimal control $\alpha_1(0,s)$ (blue) shows an excellent match to the reference control (orange dashes).}
    \label{fig:p1_validation}
\end{figure}

\begin{table}[htbp]
\caption{NHO performance on Problem 1 at the origin for $d=50$.}
\label{tab:results1}
\centering
\begin{tabular}{cccc}
\hline
\textbf{Metric} & \textbf{NHO Solution} & \textbf{Reference} & \textbf{Relative Error} \\ \hline
$V(0,\mathbf{0})$ & 3.2391 & 3.2387 & 0.012\% \\ \hline
\end{tabular}
\end{table}

\begin{figure}[htbp]
    \centering
    \includegraphics[width=\textwidth]{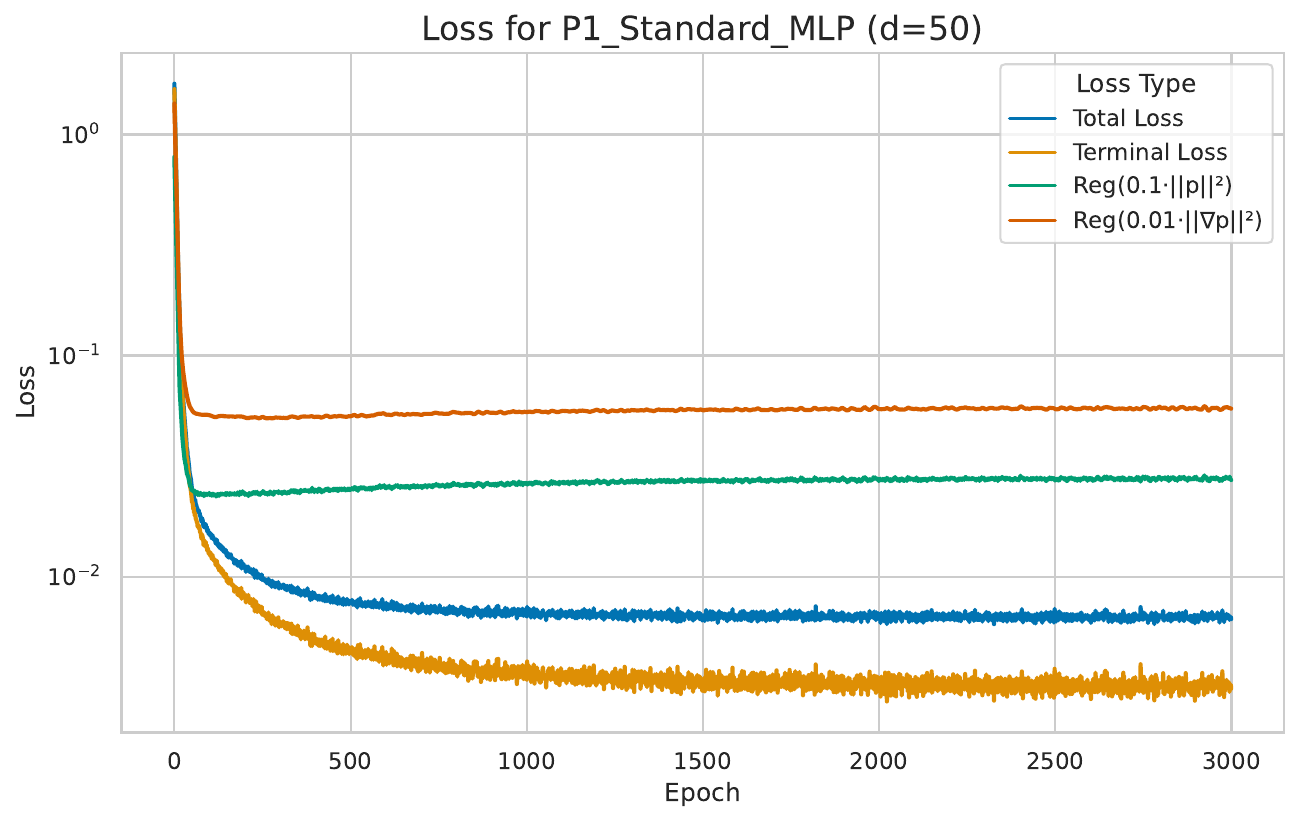} 
    \caption{Training loss components for Problem 1 ($d=50$). The terminal loss is successfully minimized, and regularization terms remain stable, indicating a well-posed optimization.}
    \label{fig:p1_loss}
\end{figure}

\subsection{Problem 2: Control in a High-Dimensional Double-Well Potential}
Next, we tackle a problem with non-trivial state-dependent drift and a non-convex value function, for which no analytical solution exists in high dimensions. The state $S_t \in \R^d$ evolves according to:
$$ dS_t = \left(-\nabla U(S_t) + \alpha_t\right) dt + \sqrt{2} d\mathcal{Z}_t, $$
where $U(s) = \frac{1}{d} \sum_{i=1}^d \frac{1}{4}(s_i^2 - 1)^2$. The objective is to minimize the cost $V(t,s) = \inf_{\alpha \in \A_t} \E_{t,s} [ \int_t^T \frac{1}{2}\norm*{\alpha_u}^2 du + U(S_T) ]$ with horizon $T=0.5$.

\subsubsection*{Results}
The NHO method successfully captures the challenging non-convex structure of the value function in $d=50$. Figure \ref{fig:p2_validation} (left) shows the characteristic "W"-shape of the value function, with minima at the stable equilibria ($s_1=\pm 1$) and a local maximum at the unstable equilibrium ($s_1=0$). The learned control, shown in the right panel, confirms this stabilizing behavior: it is positive for $s_1 < 0$ and negative for $s_1 > 0$, always acting to push the system away from the potential barrier at the origin. The learned values at these key points are reported in Table \ref{tab:results2}. The training process, shown in Figure \ref{fig:p2_loss}, is stable and effective.

\begin{figure}[htbp]
    \centering
    \includegraphics[width=\textwidth]{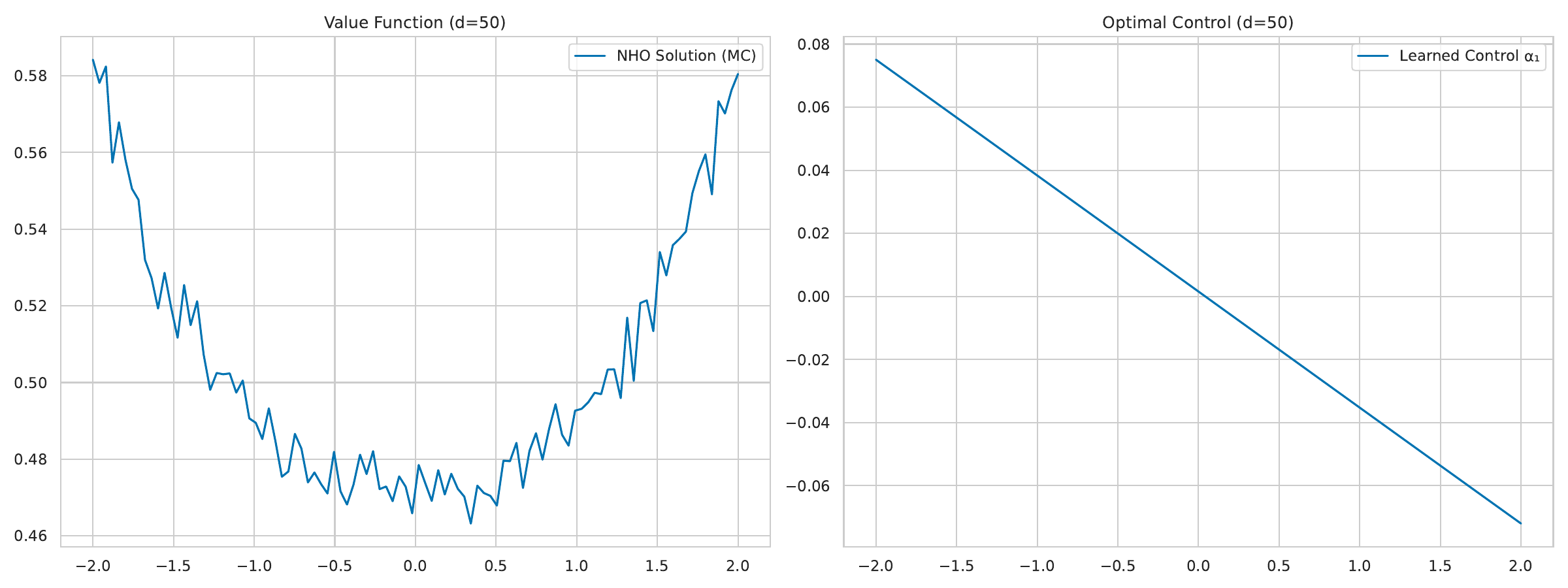} 
    \caption{Validation for Problem 2 ($d=50$). \textbf{Left:} The learned value function slice exhibits the expected non-convex, double-well shape. \textbf{Right:} The learned optimal control acts to stabilize the system by pushing it away from the unstable equilibrium at the origin.}
    \label{fig:p2_validation}
\end{figure}

\begin{table}[htbp]
\caption{NHO performance on Problem 2 ($d=50$) at key points.}
\label{tab:results2}
\centering
\begin{tabular}{cc}
\hline
\textbf{Point $s=(s_1, 0, \dots, 0)$} & \textbf{NHO Solution $V(0,s)$} \\ \hline
$s_1=0.0$ (Unstable equilibrium) & 0.177 \\
$s_1=\pm 1.0$ (Stable equilibria)  & 0.012 \\ \hline
\end{tabular}
\end{table}

\begin{figure}[htbp]
    \centering
    \includegraphics[width=\textwidth]{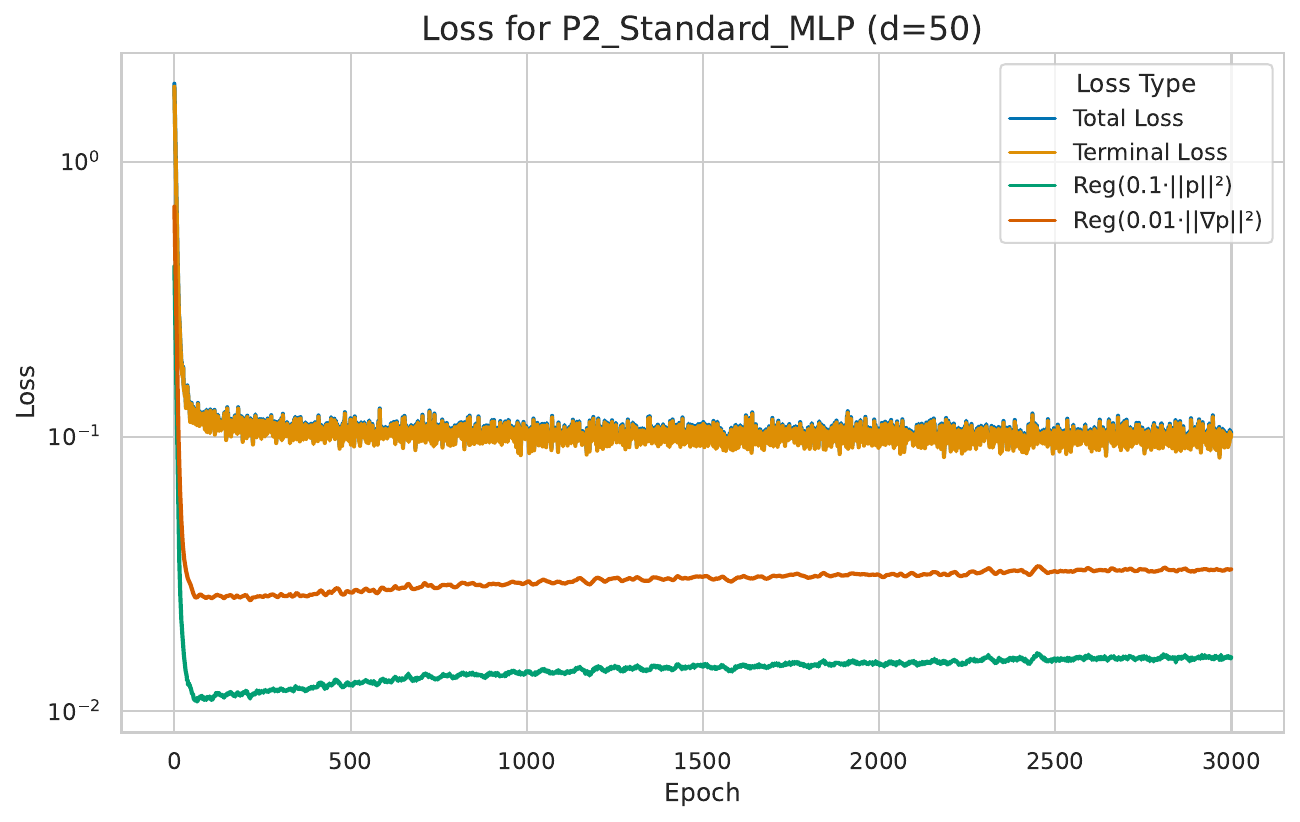} 
    \caption{Training loss for Problem 2 ($d=50$). The terminal loss is driven to a small value, indicating the PMP conditions are satisfied with high accuracy.}
    \label{fig:p2_loss}
\end{figure}

\subsection{Problem 3: Optimal Portfolio Liquidation}
Finally, we test the NHO on a classic problem from quantitative finance. The state $S_u \in \R^d$ (shares held) has dynamics $dS_u = -\alpha_u du + \sigma d\mathcal{Z}_u$, where $\alpha_u$ is the liquidation rate. The objective is to minimize total cost:
\begin{equation*}
V(t,s) = \inf_{\alpha \in \A_t} \E_{t,s} \left[ \int_t^T \kappa \sum_{i=1}^d \abs*{\alpha_{u,i}}^{3/2} du + \lambda \norm*{S_T}^2 \right].
\end{equation*}
We solve this for $d=50$ with $T=1, \kappa=0.1, \lambda=100, \sigma=0.1$, and initial position $S_0=(1,\dots,1)^\top$.

\begin{remark}[Robustness to Non-Smoothness]
The running cost function $f(\alpha) \propto \sum \abs{\alpha_{i}}^{3/2}$ is not differentiable at $\alpha=0$. This formally violates Assumption \ref{ass:main_assumptions}. Our implementation implicitly handles this by solving a smoothed version with cost $(\alpha_i^2 + \epsilon^2)^{3/4}$ for a small $\epsilon > 0$. The method's success highlights its practical robustness.
\end{remark}

\subsubsection*{Results}
The NHO solver finds a non-trivial and economically intuitive liquidation strategy. Figure \ref{fig:p3_validation} shows the expected path of the number of shares held for a single asset. The agent sells aggressively at the beginning and tapers off as the deadline approaches, exhibiting the characteristic concave shape of optimal execution strategies that balance market impact costs against the risk of a large terminal position. The key metrics of the solution are summarized in Table \ref{tab:results3}. The loss curve in Figure \ref{fig:p3_loss} shows a dramatic decrease in the total loss, driven by the terminal penalty term, demonstrating that the agent learns to effectively liquidate its portfolio.

\begin{figure}[htbp]
    \centering
    \includegraphics[width=\textwidth]{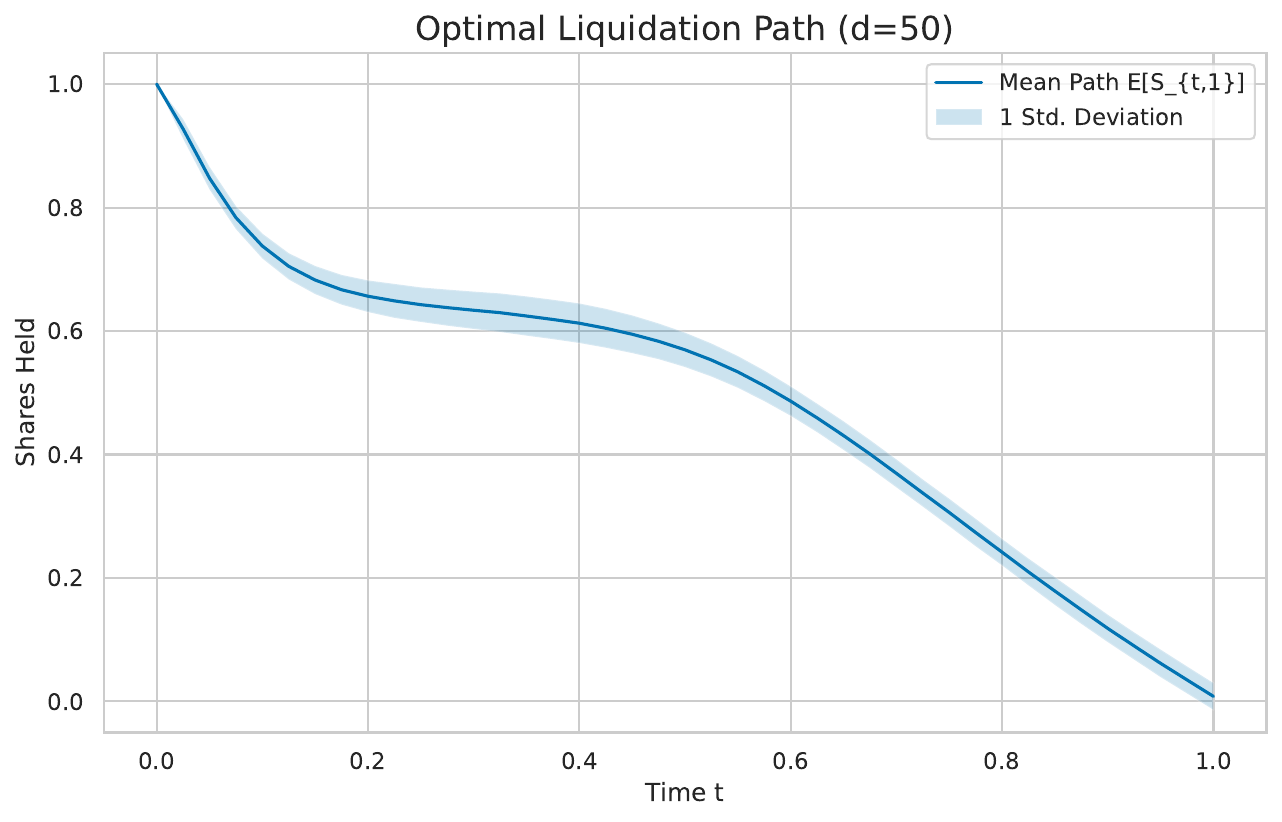} 
    \caption{Expected liquidation path $\E[S_{t,1}]$ for a single asset in the $d=50$ portfolio. The solid blue line shows the sample mean of shares held, with the shaded region representing one standard deviation. The path's concavity is characteristic of optimal execution strategies.}
    \label{fig:p3_validation}
\end{figure}

\begin{table}[htbp]
\caption{NHO results for the Optimal Portfolio Liquidation problem ($d=50$).}
\label{tab:results3}
\centering
\begin{tabular}{ccc}
\hline
\textbf{Metric} & \textbf{Value} & \textbf{Interpretation} \\ \hline
Initial Shares ($S_{0,i}$) & 1.0 & Starting position per asset \\
Final Shares $\E[S_{T,i}]$ & 0.003 & Agent successfully liquidates ~99.7\% \\
Total Cost $V(0,S_0)$ & 6.27 & Learned optimal cost from initial state \\
\hline
\end{tabular}
\end{table}

\begin{figure}[htbp]
    \centering
    \includegraphics[width=\textwidth]{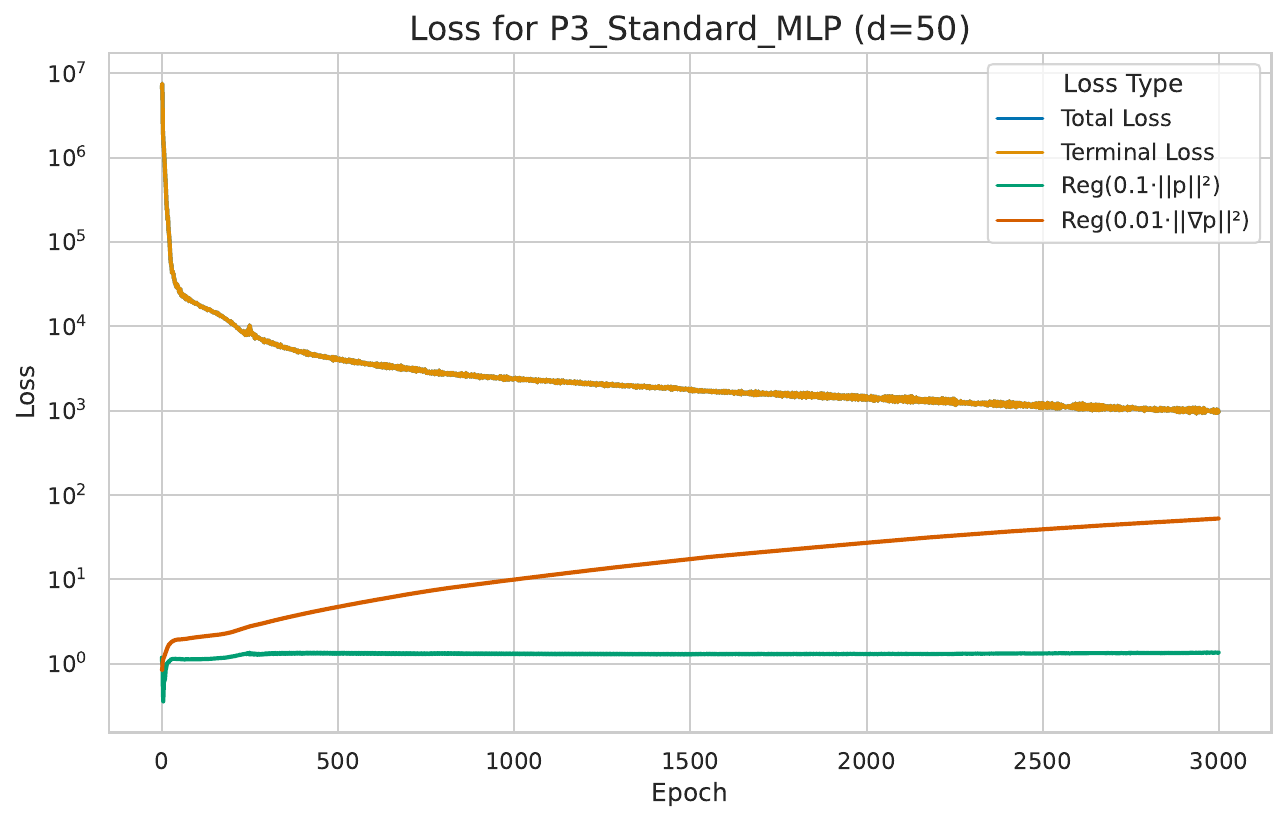} 
    \caption{Training loss for Problem 3 ($d=50$). The total loss (dominated by the terminal penalty) decreases by several orders of magnitude, showing the agent learns an effective liquidation policy. The gradient regularization term increases as a more complex control is learned.}
    \label{fig:p3_loss}
\end{figure}

\begin{remark}[On the Choice of Activation Function]
In accordance with the theoretical requirements of Theorem \ref{thm:approx} for ensuring the existence of the operator coefficients, which depend on the network's Jacobian, we employ a smooth activation function in all our experiments. Specifically, we use the hyperbolic tangent (\textit{tanh}) function. This choice directly satisfies the differentiability conditions of Lemma \ref{lem:c1_uat_full} and ensures that our numerical implementation is consistent with the theory that underpins it.
\end{remark}

\section{Conclusion}
\label{sec:conclusion}
We have introduced the Neural Hamiltonian Operator (NHO) as a formal mathematical object to structure the application of deep learning to high-dimensional stochastic control. This operator-theoretic view recasts the learning problem as a search for an optimal dynamical generator within a parameterized class. From a statistical perspective, this is a problem of non-parametric inference: learning an operator that represents an unknown dynamical law, using data generated under a model class constrained by physical principles. Our framework generalizes previous work to systems driven by continuous martingales, broadening its applicability.

We have proven that the NHO family possesses the crucial property of universal approximation (Theorem \ref{thm:approx}), providing a solid theoretical foundation for the approach. We have further validated this promise with numerical experiments that solve challenging nonlinear control problems in up to 100 dimensions. The remaining challenges, while still significant, can now be framed more clearly as questions of optimization and stability, which are central to modern statistical learning theory. We analyzed convergence under both idealized global conditions (Theorem \ref{thm:convergence}) and more practical local conditions aided by regularization (Theorem \ref{thm:local_convergence}). We believe that investigating the mathematical properties of Neural Hamiltonian Operators is a crucial and exciting direction for future research at the intersection of mathematics, statistics, and artificial intelligence.

\bibliographystyle{imsart-nameyear}
\bibliography{main}

\begin{appendix}

\section{Proof of the \texorpdfstring{C\textsuperscript{1}}{C1}-Approximation Lemma}
\label{app:c1_proof}

The result that neural networks can approximate not only a function but also its derivatives simultaneously is a cornerstone of their application to solving differential equations. We provide a detailed proof for completeness, following the classical approach which combines a mollification argument with a density theorem for neural networks.

\begin{lemma}[$C^1$-Universal Approximation]
\label{lem:c1_uat_full}
Let $K \subset \R^n$ be a compact set and let $h: K \to \R^m$ be a $C^1$ function (i.e., continuously differentiable). Let the neural network architecture use an activation function $\phi \in C^\infty(\R)$ that is not a polynomial (e.g., $\tanh(x)$ or the sigmoid function). Then for any $\delta > 0$, there exists a single-hidden-layer neural network $\hat{h}_\theta(x) = \sum_{i=1}^N \bm{c}_i \phi(w_i^\top x + b_i)$ with parameters $\theta = \{\bm{c}_i, w_i, b_i\}$, where $\bm{c}_i \in \R^m$, $w_i \in \R^n$, $b_i \in \R$, such that the function and its Jacobian are uniformly close to those of $h$. Specifically, they are close in the $C^1(K)$ norm, defined as $\|f\|_{C^1(K)} = \sup_{x \in K} \norm*{f(x)} + \sup_{x \in K} \norm*{\nabla f(x)}_F$:
$$
\norm*{\hat{h}_\theta - h}_{C^1(K)} < \delta.
$$
\end{lemma}
\begin{proof}
The proof proceeds in three steps. First, we extend the function $h$ from the compact set $K$ to all of $\R^n$. Second, we approximate this extended function with an infinitely differentiable function using a mollifier. Third, we approximate this smooth function and its derivative with a neural network.

\textbf{Step 1: Extension to $\R^n$.}
Since $h$ is $C^1$ on a compact set $K$, it can be extended to a $C^1$ function on all of $\R^n$ which has compact support. This is a standard result from analysis (a consequence of the Whitney extension theorem). Let this extension, which we still denote by $h$, be supported on a larger compact set $K' \supset K$.

\textbf{Step 2: Approximation by an infinitely differentiable function (Mollification).}
Let $\eta: \R^n \to \R$ be a standard mollifier, i.e., a function satisfying:
(i) $\eta \in C^\infty(\R^n)$, (ii) $\eta(x) \ge 0$ for all $x$, (iii) $\eta$ has compact support (e.g., within the unit ball), and (iv) $\int_{\R^n} \eta(x) dx = 1$.
For any $\epsilon > 0$, we define the scaled mollifier $\eta_\epsilon(x) = \epsilon^{-n} \eta(x/\epsilon)$. We then define the mollified function $h_\epsilon$ as the convolution of $h$ with $\eta_\epsilon$:
$$ h_\epsilon(x) = (h * \eta_\epsilon)(x) = \int_{\R^n} h(x-y) \eta_\epsilon(y) dy. $$
A fundamental property of convolutions is that if one function is $C^k$ and the other is $C^\infty$, the result is $C^\infty$. Since $h \in C^1$ and $\eta_\epsilon \in C^\infty$, it follows that $h_\epsilon \in C^\infty(\R^n)$. Furthermore, standard analysis shows that as $\epsilon \to 0$, $h_\epsilon$ converges to $h$ in the $C^1$ norm on any compact set. This is because differentiation commutes with convolution: $\nabla h_\epsilon = (\nabla h) * \eta_\epsilon$, and convolution with an approximate identity like $\eta_\epsilon$ converges to the identity map.
Thus, for any $\delta' > 0$, we can choose $\epsilon$ small enough such that:
$$ \norm*{h_\epsilon - h}_{C^1(K)} = \sup_{x \in K} \norm*{h_\epsilon(x) - h(x)} + \sup_{x \in K} \norm*{\nabla h_\epsilon(x) - \nabla h(x)}_F < \frac{\delta}{2}. $$

\textbf{Step 3: Density of neural networks in $C^1(K)$.}
The core result, established by \cite{Hornik1990}, is that single-hidden-layer feedforward networks with a smooth, non-polynomial activation function are dense in the space $C^k(K)$ for any $k \ge 0$, equipped with the topology of uniform convergence for the function and all its derivatives up to order $k$. For our case ($k=1$), this means that for the smooth function $h_\epsilon$ from Step 2, there exist network parameters $\theta$ such that the network $\hat{h}_\theta$ satisfies:
$$ \norm*{\hat{h}_\theta - h_\epsilon}_{C^1(K)} < \frac{\delta}{2}. $$
The Jacobian of the network is explicitly given by $\nabla \hat{h}_\theta(x) = \sum_{i=1}^N \bm{c}_i \phi'(w_i^\top x + b_i) w_i^\top$. This is well-defined because we assumed the activation function $\phi$ is at least $C^1$.

\textbf{Step 4: Combining the results via the Triangle Inequality.}
We can now bound the total approximation error using the triangle inequality for the $C^1(K)$ norm:
\begin{align*}
\norm*{\hat{h}_\theta - h}_{C^1(K)} &\le \norm*{\hat{h}_\theta - h_\epsilon}_{C^1(K)} + \norm*{h_\epsilon - h}_{C^1(K)} \\
&< \frac{\delta}{2} + \frac{\delta}{2} = \delta.
\end{align*}
This completes the proof, showing that for any $h \in C^1(K)$ and any tolerance $\delta>0$, a neural network $\hat{h}_\theta$ exists that is $\delta$-close in the $C^1$ sense.
\end{proof}

\section{Proof of Theorem \ref{thm:convergence} (Global Convergence)}
\label{app:global_conv_proof}

This proof provides a detailed derivation of the convergence of Stochastic Gradient Descent (SGD) under the Polyak-Łojasiewicz (P-L) condition.

\begin{proof}
Let $\Psi_k$ be the parameters at step $k$. The SGD update is $\Psi_{k+1} = \Psi_k - \gamma_k \hat{g}_k$, where $\hat{g}_k \equiv \hat{g}(\Psi_k)$ is the stochastic gradient.

\textbf{Step 1: The Descent Lemma from L-Smoothness.}
Assumption 2 states that $\mathcal{J}(\Psi)$ is L-smooth. By definition, this means its gradient $\nabla\mathcal{J}$ is Lipschitz continuous with constant $L$: $\norm{\nabla\mathcal{J}(\Psi_a) - \nabla\mathcal{J}(\Psi_b)} \le L \norm{\Psi_a - \Psi_b}$. A direct consequence of this is the descent lemma. We derive it here. By the Fundamental Theorem of Calculus for vector-valued functions, for any $\Psi_a, \Psi_b$:
$$ \mathcal{J}(\Psi_b) - \mathcal{J}(\Psi_a) = \int_0^1 \langle \nabla \mathcal{J}(\Psi_a + t(\Psi_b - \Psi_a)), \Psi_b - \Psi_a \rangle dt. $$
Subtracting $\langle \nabla \mathcal{J}(\Psi_a), \Psi_b - \Psi_a \rangle$ from both sides gives:
\begin{align*}
\mathcal{J}(\Psi_b) - \mathcal{J}(\Psi_a) - \langle \nabla \mathcal{J}(\Psi_a), \Psi_b - \Psi_a \rangle &= \int_0^1 \langle \nabla \mathcal{J}(\Psi_a + t(\Psi_b - \Psi_a)) - \nabla \mathcal{J}(\Psi_a), \Psi_b - \Psi_a \rangle dt \\
&\le \int_0^1 \norm{\nabla \mathcal{J}(\Psi_a + t(\Psi_b - \Psi_a)) - \nabla \mathcal{J}(\Psi_a)} \norm{\Psi_b - \Psi_a} dt  \\
&\le \int_0^1 L \norm{t(\Psi_b - \Psi_a)} \norm{\Psi_b - \Psi_a} dt \quad \text{(L-smoothness)} \\
&= L \norm{\Psi_b - \Psi_a}^2 \int_0^1 t dt = \frac{L}{2} \norm{\Psi_b - \Psi_a}^2.
\end{align*}
Rearranging gives the descent lemma: $\mathcal{J}(\Psi_b) \le \mathcal{J}(\Psi_a) + \langle \nabla \mathcal{J}(\Psi_a), \Psi_b - \Psi_a \rangle + \frac{L}{2} \norm{\Psi_b - \Psi_a}^2$.
Applying this to the SGD update by setting $\Psi_a = \Psi_k$ and $\Psi_b = \Psi_{k+1}$:
$$ \mathcal{J}(\Psi_{k+1}) \le \mathcal{J}(\Psi_k) - \gamma_k \langle \nabla \mathcal{J}(\Psi_k), \hat{g}_k \rangle + \frac{L \gamma_k^2}{2} \norm{\hat{g}_k}^2. $$

\textbf{Step 2: Taking Conditional Expectation.}
Let $\mathcal{F}_k = \sigma(\Psi_0, \hat{g}_0, \dots, \Psi_k)$ be the sigma-algebra generated by the optimization history up to step $k$. We take the conditional expectation of the above inequality with respect to $\mathcal{F}_k$.
$$ \E[\mathcal{J}(\Psi_{k+1}) | \mathcal{F}_k] \le \mathcal{J}(\Psi_k) - \gamma_k \E[\langle \nabla \mathcal{J}(\Psi_k), \hat{g}_k \rangle | \mathcal{F}_k] + \frac{L \gamma_k^2}{2} \E[\norm{\hat{g}_k}^2 | \mathcal{F}_k]. $$
Since $\Psi_k$ is $\mathcal{F}_k$-measurable, so is $\nabla \mathcal{J}(\Psi_k)$. We can thus pull it out of the inner product's expectation. By Assumption 2, the stochastic gradient is unbiased, $\E[\hat{g}_k | \mathcal{F}_k] = \nabla \mathcal{J}(\Psi_k)$.
$$ \E[\langle \nabla \mathcal{J}(\Psi_k), \hat{g}_k \rangle | \mathcal{F}_k] = \langle \nabla \mathcal{J}(\Psi_k), \E[\hat{g}_k | \mathcal{F}_k] \rangle = \norm{\nabla \mathcal{J}(\Psi_k)}^2. $$
Also by Assumption 2, the variance of the stochastic gradient is bounded, $\E[\norm{\hat{g}_k - \nabla\mathcal{J}(\Psi_k)}^2] \le M_v$. This implies that the expected squared norm is bounded: $\E[\norm{\hat{g}_k}^2 | \mathcal{F}_k] \le \norm{\nabla\mathcal{J}(\Psi_k)}^2 + M_v$. A simpler, sufficient condition often used (and implied by Assumption 2 in the paper text) is a uniform bound on the second moment: $\E[\norm{\hat{g}_k}^2 | \mathcal{F}_k] \le M$ for some constant $M$. Using this gives:
$$ \E[\mathcal{J}(\Psi_{k+1}) | \mathcal{F}_k] \le \mathcal{J}(\Psi_k) - \gamma_k \norm{\nabla \mathcal{J}(\Psi_k)}^2 + \frac{L M \gamma_k^2}{2}. $$

\textbf{Step 3: Applying the Polyak-Łojasiewicz (P-L) Condition.}
Assumption 3 provides the P-L condition: $\norm{\nabla \mathcal{J}(\Psi)}^2 \ge 2c \mathcal{J}(\Psi)$ for some $c>0$. Under this condition, all stationary points are global minima. Substituting this into our inequality:
$$ \E[\mathcal{J}(\Psi_{k+1}) | \mathcal{F}_k] \le \mathcal{J}(\Psi_k) - 2c\gamma_k \mathcal{J}(\Psi_k) + \frac{L M \gamma_k^2}{2} = (1 - 2c\gamma_k) \mathcal{J}(\Psi_k) + \frac{L M \gamma_k^2}{2}. $$

\textbf{Step 4: Convergence via Robbins-Siegmund Lemma.}
Let $a_k = \E[\mathcal{J}(\Psi_k)]$. Taking the full expectation of the inequality above and using the tower property ($\E[\cdot] = \E[\E[\cdot|\mathcal{F}_k]]$), we obtain a recurrence for the non-negative sequence $\{a_k\}$:
$$ a_{k+1} \le (1 - 2c\gamma_k) a_k + \frac{L M \gamma_k^2}{2}. $$
This is a supermartingale-like recurrence to which the Robbins-Siegmund lemma applies. The lemma states that for a non-negative sequence $\{z_k\}$ and sequences $\{\alpha_k\}, \{\beta_k\}$ satisfying $z_{k+1} \le (1+\alpha_k)z_k + \beta_k$ with $\sum \alpha_k < \infty$ and $\sum \beta_k < \infty$, $z_k$ converges almost surely. A more specific version for our case states that if $a_{k+1} \le (1 - \zeta_k) a_k + \eta_k$ with $a_k \ge 0$, $\zeta_k \in [0,1]$, $\sum \zeta_k = \infty$, and $\sum \eta_k < \infty$, then $\lim_{k\to\infty} a_k = 0$.
We identify $\zeta_k = 2c\gamma_k$ and $\eta_k = \frac{L M \gamma_k^2}{2}$. For $k$ large enough, $\gamma_k$ is small, so $\zeta_k \in [0,1]$. The learning rate conditions $\sum_{k=0}^\infty \gamma_k = \infty$ and $\sum_{k=0}^\infty \gamma_k^2 < \infty$ directly imply that $\sum \zeta_k = \infty$ and $\sum \eta_k < \infty$.
Therefore, all conditions of the lemma are met, and we conclude that $\lim_{k \to \infty} a_k = \lim_{k \to \infty} \E[\mathcal{J}(\Psi_k)] = 0$.
\end{proof}

\section{Proof of Theorem \ref{thm:ergodic_consistency} (Ergodic Consistency)}
\label{app:ergodic_proof}

This proof provides a detailed argument showing that if the ergodic loss is minimized to zero, the learned Hamiltonian must be constant with respect to the system's invariant measure.

\begin{proof}
\textbf{Step 1: From Zero Loss to Pathwise Constant Hamiltonian.}
The ergodic loss is defined as $\mathcal{J}_{ergodic}(\Psi) = \E [ \frac{1}{T} \int_0^T ( \hamiltonian_\Psi(S_t) - \bar{\hamiltonian}_{\Psi,T} )^2 dt ]$. By Assumption 3, the optimization finds $\Psi^*$ such that $\mathcal{J}_{ergodic}(\Psi^*) = 0$. The expression inside the expectation is a non-negative random variable (being an integral of a squared term). The expectation of a non-negative random variable is zero if and only if the random variable itself is zero almost surely (with respect to the probability measure $\mathbb{P}$ on the path space $\Omega$). Thus, for $\mathbb{P}$-almost every path $\omega \in \Omega$:
$$ \frac{1}{T} \int_0^T \left( \hamiltonian_{\Psi^*}(S_t(\omega)) - \bar{\hamiltonian}_{\Psi^*,T}(\omega) \right)^2 dt = 0. $$
The integrand, let's call it $g(t) = (\hamiltonian_{\Psi^*}(S_t(\omega)) - \bar{\hamiltonian}_{\Psi^*,T}(\omega))^2$, is a continuous function of time $t$. This is because $S_t(\omega)$ is a continuous path (as it solves an SDE driven by a continuous martingale) and $\hamiltonian_{\Psi^*}(s)$ is a continuous function of state $s$ (by Assumption 2). The integral of a non-negative continuous function over an interval is zero if and only if the function is identically zero on that interval. Therefore, for $\mathbb{P}$-almost every $\omega$, we must have:
$$ \hamiltonian_{\Psi^*}(S_t(\omega)) - \bar{\hamiltonian}_{\Psi^*,T}(\omega) = 0 \quad \text{for all } t \in [0,T]. $$
This implies that for almost every sample path, the function $t \mapsto \hamiltonian_{\Psi^*}(S_t(\omega))$ is a constant. Let us denote this path-dependent constant by $C(\omega) \coloneqq \bar{\hamiltonian}_{\Psi^*,T}(\omega)$.

\textbf{Step 2: Invoking the Birkhoff Ergodic Theorem.}
By Assumption 1, the state process $S_t$ generated by the operator $L_{\Psi^*}$ is ergodic with a unique invariant probability measure $\pi_{\Psi^*}$. Let $h(s) \coloneqq \hamiltonian_{\Psi^*}(s)$. This function is continuous (Assumption 2) and thus measurable. We assume it is integrable with respect to the invariant measure, i.e., $\int |h(s)| d\pi_{\Psi^*}(s) < \infty$.
The Birkhoff Pointwise Ergodic Theorem states that for any such integrable function $h$ and for $\mathbb{P}$-almost every path $\omega$, the time average of $h$ along the trajectory converges to the spatial average of $h$ with respect to the invariant measure:
$$ \lim_{T \to \infty} \frac{1}{T} \int_0^T h(S_t(\omega)) dt = \int_{\R^d} h(s) d\pi_{\Psi^*}(s). $$
Let us define the spatial average (which is a deterministic constant) as $\lambda^* \coloneqq \int_{\R_d} \hamiltonian_{\Psi^*}(s) d\pi_{\Psi^*}(s)$.

\textbf{Step 3: Unifying the Path-Dependent Constants.}
From Step 1, we know that for a sufficiently large but finite $T$, for $\mathbb{P}$-a.e. $\omega$, the time average $\frac{1}{T} \int_0^T \hamiltonian_{\Psi^*}(S_t(\omega)) dt$ is simply the constant $C(\omega)$. Combining this with the result from Step 2, we have that for $\mathbb{P}$-a.e. $\omega$:
$$ C(\omega) = \lambda^*. $$
This is a critical deduction: the constant value that the Hamiltonian takes along a specific path must be the same for almost all paths, and this universal constant is precisely the spatial average of the Hamiltonian function over the invariant measure.

\textbf{Step 4: From Pathwise Constancy to State-Space Constancy.}
We have established that for $\mathbb{P}$-a.e. path $\omega$, $\hamiltonian_{\Psi^*}(S_t(\omega)) = \lambda^*$ for all $t \ge 0$. Let $\Omega_0 \subset \Omega$ be the set of such paths, with $\mathbb{P}(\Omega_0)=1$. Let $\mathcal{A} = \{ S_t(\omega) \mid t \ge 0, \omega \in \Omega_0 \}$ be the set of all states visited by these "good" paths. By construction, for any state $s \in \mathcal{A}$, we have $\hamiltonian_{\Psi^*}(s) = \lambda^*$.
A key property of an ergodic process is that the set of states visited by almost all paths has full measure under the invariant distribution. That is, $\pi_{\Psi^*}(\mathcal{A}) = 1$.
So, we have a continuous function $\hamiltonian_{\Psi^*}(s)$ that is equal to a constant $\lambda^*$ on a set $\mathcal{A}$ of full measure $\pi_{\Psi^*}(\mathcal{A}) = 1$. This implies that the function is equal to $\lambda^*$ for $\pi_{\Psi^*}$-almost every $s \in \R^d$.
\end{proof}

\section{Proof of Theorem \ref{thm:local_convergence} (Local Convergence)}
\label{app:local_conv_proof}

This proof details the convergence of SGD in a neighborhood of a local minimum satisfying a local version of the P-L condition.

\begin{proof}
Let $\Psi^*$ be the parameters of a local minimum of the regularized loss $\mathcal{J}_\lambda$, and let $\mathcal{J}^*_\lambda = \mathcal{J}_\lambda(\Psi^*)$. The proof structure mirrors that of the global convergence proof, but the arguments are restricted to a neighborhood $\mathcal{N}(\Psi^*)$. We assume the SGD iterates $\Psi_k$ remain within this neighborhood.

\textbf{Step 1: Recurrence for the Excess Loss.}
By Assumption \ref{ass:local_structure}.1, the loss is L-smooth in $\mathcal{N}(\Psi^*)$. Thus, the descent lemma holds for any $\Psi_k \in \mathcal{N}(\Psi^*)$, and after taking conditional expectations, we have the same inequality as in the global case:
$$ \E[\mathcal{J}_\lambda(\Psi_{k+1}) | \mathcal{F}_k] \le \mathcal{J}_\lambda(\Psi_k) - \gamma_k \norm{\nabla \mathcal{J}_\lambda(\Psi_k)}^2 + \frac{L M \gamma_k^2}{2}. $$
Let $a_k = \E[\mathcal{J}_\lambda(\Psi_k) - \mathcal{J}^*_\lambda]$ be the expected "excess loss" relative to the optimal value in the neighborhood. Subtracting the constant $\mathcal{J}^*_\lambda$ from both sides of the inequality does not change it. Taking the full expectation yields:
$$ \E[\mathcal{J}_\lambda(\Psi_{k+1})] - \mathcal{J}^*_\lambda \le \E[\mathcal{J}_\lambda(\Psi_k)] - \mathcal{J}^*_\lambda - \gamma_k \E[\norm{\nabla \mathcal{J}_\lambda(\Psi_k)}^2] + \frac{L M \gamma_k^2}{2}. $$
$$ a_{k+1} \le a_k - \gamma_k \E[\norm{\nabla \mathcal{J}_\lambda(\Psi_k)}^2] + \frac{L M \gamma_k^2}{2}. $$

\textbf{Step 2: Applying the Local P-L Condition.}
By Assumption \ref{ass:local_structure}.2, the local P-L condition holds within $\mathcal{N}(\Psi^*)$: $\norm{\nabla \mathcal{J}_\lambda(\Psi)}^2 \ge 2c (\mathcal{J}_\lambda(\Psi) - \mathcal{J}^*_\lambda)$. Taking the expectation:
$$ \E[\norm{\nabla \mathcal{J}_\lambda(\Psi_k)}^2] \ge 2c \, \E[\mathcal{J}_\lambda(\Psi_k) - \mathcal{J}^*_\lambda] = 2c \, a_k. $$
Substituting this into our recurrence for $a_k$:
\begin{align*}
a_{k+1} &\le a_k - \gamma_k (2c \, a_k) + \frac{L M \gamma_k^2}{2} \\
a_{k+1} &\le (1 - 2c\gamma_k) a_k + \frac{L M \gamma_k^2}{2}.
\end{align*}

\textbf{Step 3: Convergence via Robbins-Siegmund.}
This recurrence for the non-negative sequence $a_k$ is identical in form to the one derived in the global convergence proof (Appendix \ref{app:global_conv_proof}). With learning rates satisfying $\sum \gamma_k = \infty$ and $\sum \gamma_k^2 < \infty$, the conditions of the Robbins-Siegmund lemma are satisfied.
Therefore, we conclude that $\lim_{k\to\infty} a_k = 0$. This implies $\lim_{k \to \infty} \E[\mathcal{J}_\lambda(\Psi_k) - \mathcal{J}^*_\lambda] = 0$, or $\lim_{k \to \infty} \E[\mathcal{J}_\lambda(\Psi_k)] = \mathcal{J}_\lambda(\Psi^*)$.

The crucial caveat is that this proof relies on the assumption that the iterates remain within the neighborhood $\mathcal{N}(\Psi^*)$. This is a practical consideration: if the algorithm is initialized sufficiently close to a "good" local minimum (one satisfying the local P-L condition) and the learning rates are sufficiently small, the updates are unlikely to propel the parameters outside this region of well-behaved geometry.
\end{proof}

\section{Proof of Theorem \ref{thm:ergodic_stability} (Ergodic Stability)}
\label{app:ergodic_stability_proof}

This proof formally establishes the ergodicity of the learned system by applying the theory of stochastic stability, specifically the Foster-Lyapunov drift criterion.

\begin{proof}
\textbf{Step 1: Deriving the Drift of the Lyapunov Function via Itô's Formula.}
We choose the quadratic function $U(s) = \norm{s}^2$ as our candidate Lyapunov function. Let $S_t$ be the state process generated by the learned operator $L_{\Psi^*,S}$, with drift $\mu_{\Psi^*}(s) = \mu(s, \alpha_{\omega^*}(s))$ and diffusion matrix $\sigma_{\Psi^*}(s) = \sigma(s, \alpha_{\omega^*}(s))$. The SDE is $dS_t = \mu_{\Psi^*}(S_t) dt + \sigma_{\Psi^*}(S_t) d\mathcal{Z}_t$ (specializing to Brownian motion for clarity, as $C=I$).
We apply Itô's formula to the process $U(S_t)$:
$$ dU(S_t) = \nabla U(S_t)^\top dS_t + \frac{1}{2} \Tr\left( (\sigma_{\Psi^*}(S_t) \sigma_{\Psi^*}(S_t)^\top) \nabla^2 U(S_t) \right) dt. $$
The derivatives of $U(s)$ are $\nabla U(s) = 2s$ and $\nabla^2 U(s) = 2I_d$. Substituting these and the SDE for $dS_t$:
\begin{align*}
dU(S_t) &= (2S_t)^\top (\mu_{\Psi^*}(S_t) dt + \sigma_{\Psi^*}(S_t) d\mathcal{Z}_t) + \frac{1}{2} \Tr\left( (\sigma_{\Psi^*}(S_t) \sigma_{\Psi^*}(S_t)^\top) (2I_d) \right) dt \\
&= \left[ 2S_t^\top \mu_{\Psi^*}(S_t) + \Tr(\sigma_{\Psi^*}(S_t) \sigma_{\Psi^*}(S_t)^\top) \right] dt + 2S_t^\top \sigma_{\Psi^*}(S_t) d\mathcal{Z}_t.
\end{align*}
The term in the square brackets is precisely the action of the generator $L_{\Psi^*,S}$ on the function $U$, i.e., $(L_{\Psi^*,S} U)(S_t)$. So, we have $dU(S_t) = (L_{\Psi^*,S} U)(S_t) dt + (\text{martingale term})$.

\textbf{Step 2: Applying the Dissipativity Assumption.}
The theorem assumes a dissipativity condition on the underlying dynamics, which implies that for the learned control policy $\alpha_{\omega^*}(s)$, there exists a compact set $\mathcal{C} \subset \R^d$ and a constant $K>0$ such that for all $s \notin \mathcal{C}$:
$$ (L_{\Psi^*,S} U)(s) = 2s^\top \mu_{\Psi^*}(s) + \Tr(\sigma_{\Psi^*}(s)\sigma_{\Psi^*}(s)^\top) \le -K \norm{s}^2 = -K U(s). $$

\textbf{Step 3: Verifying the Foster-Lyapunov Condition.}
The inequality derived in Step 2 is a classic Foster-Lyapunov drift condition. A fundamental theorem of stochastic stability [cf. \cite{Khasminskii2011}, \cite{Meyn2009}] states that if there exists a non-negative test function $U(s)$ such that $U(s) \to \infty$ as $\norm{s} \to \infty$, and for some compact set $\mathcal{C}$, the generator's action satisfies $(LU)(s) \le -c_1$ for some $c_1>0$ for all $s \notin \mathcal{C}$, then the process is positive Harris recurrent. Our condition $(L_{\Psi^*,S} U)(s) \le -K U(s)$ is even stronger, as $U(s)$ grows with $s$.
Positive Harris recurrence implies the existence of a unique stationary probability distribution $\pi_{\Psi^*}$. If the process is also irreducible (which is guaranteed if the diffusion matrix $\sigma_{\Psi^*}\sigma_{\Psi^*}^\top$ is strictly positive definite everywhere, ensuring the process can move between any two open sets), then positive Harris recurrence is equivalent to ergodicity.

\textbf{Step 4: Role of the Optimization Objective.}
The theorem states that the optimization finds parameters $\Psi^*$ for which the time-averaged Lyapunov drift is negative:
$$ \E\left[\frac{1}{T} \int_0^T (L_{\Psi^*,S} U)(S_t) dt\right] < 0. $$
This serves as an empirical check of the stability condition. Since we have established from the dissipativity assumption that the process is ergodic, the ergodic theorem guarantees that this time average converges to the spatial average almost surely:
$$ \lim_{T\to\infty} \frac{1}{T} \int_0^T (L_{\Psi^*,S} U)(S_t) dt = \int_{\R^d} (L_{\Psi^*,S} U)(s) d\pi_{\Psi^*}(s). $$
The negative value of this integral confirms that, under the stationary measure, the system is on average dissipative, reinforcing the conclusion of stability. The Lyapunov regularizer in the loss function explicitly encourages the optimizer to find a parameter set $\Psi^*$ for which this condition holds, thereby promoting the learning of a stable system.
\end{proof}

\end{appendix}
\end{document}